\documentclass{arxiv}

\usepackage{mathtools,amsfonts,amsthm}
\DeclareMathOperator*{\argmin}{arg\,min}
\newtheorem{theorem}{Theorem}[section]
\newtheorem{lemma}[theorem]{Lemma}

\usepackage{graphicx}
\usepackage{subcaption}
\usepackage{booktabs}
\usepackage{algorithm}
\usepackage{algpseudocode}

\usepackage{microtype}
\usepackage[authoryear]{natbib}
\usepackage[pdfusetitle]{hyperref}

\title{Relative Density Ratio Optimization for Stable and Statistically Consistent Model Alignment}

\author{%
  Hiroshi Takahashi\textsuperscript{1} \And
  Tomoharu Iwata\textsuperscript{1} \And
  Atsutoshi Kumagai\textsuperscript{1} \And
  Sekitoshi Kanai\textsuperscript{1} \And
  Masanori Yamada\textsuperscript{2} \And
  Kosuke Nishida\textsuperscript{1} \And
  Kazutoshi Shinoda\textsuperscript{1}%
}
\affiliation{\textsuperscript{1}NTT, Inc. \quad \textsuperscript{2}NTT DOCOMO, INC.}
\correspondence{\href{mailto:hiroshibm.takahashi@ntt.com}{hiroshibm.takahashi@ntt.com}}
\codeurl{https://github.com/takahashihiroshi/rdro}
\hypersetup{pdfauthor={Hiroshi Takahashi, Tomoharu Iwata, Atsutoshi Kumagai, Sekitoshi Kanai, Masanori Yamada, Kosuke Nishida, Kazutoshi Shinoda}}
\renewcommand{\shorttitle}{Relative Density Ratio Optimization}

\begin{document}

\begin{abstract}
Aligning language models with human preferences is essential for ensuring their safety and reliability.
Although most existing approaches assume specific human preference models such as the Bradley-Terry model,
this assumption may fail to accurately capture true human preferences,
and consequently, these methods lack statistical consistency,
i.e., the guarantee that language models converge to the true human preference as the number of samples increases.
In contrast,
direct density ratio optimization (DDRO) achieves statistical consistency without assuming any human preference models.
DDRO models the density ratio between preferred and non-preferred data distributions using the language model,
and then optimizes it via density ratio estimation.
However, this density ratio is unstable and often diverges, leading to training instability of DDRO.
In this paper, we propose a novel alignment method that is both stable and statistically consistent.
Our approach is based on the relative density ratio between the preferred data distribution and a mixture of the preferred and non-preferred data distributions.
Our approach is stable since this relative density ratio is bounded above and does not diverge.
Moreover, it is statistically consistent and yields significantly tighter convergence guarantees than DDRO.
We experimentally show its effectiveness with Qwen 2.5 and Llama 3.
\end{abstract}

\maketitle

\section{Introduction}
\label{sec:introduction}

Language models have demonstrated strong performance in text generation, translation, question answering,
and information retrieval \citep{vaswani2017attention,brown2020language,achiam2023gpt}.
Nevertheless, they can produce inaccurate, harmful, or biased outputs \citep{ji2023survey,bender2021dangers,sheng2019woman,gehman2020realtoxicityprompts},
which may threaten their safety and reliability.

Alignment, which tunes language models to follow human preferences using a preference dataset containing preferred and non-preferred responses to various prompts,
is essential for ensuring safety and reliability \citep{ouyang2022training,rafailov2023direct,ethayarajh2024kto}.
To model human preferences,
most existing approaches assume the Bradley-Terry model \citep{bradley1952rank},
which is widely used in learning to rank.
For example, reinforcement learning from human feedback (RLHF) \citep{ouyang2022training} first trains a reward model using the preference dataset under the Bradley-Terry assumption,
and then tunes the language model via reinforcement learning to maximize the reward-model score.
Direct preference optimization (DPO) \citep{rafailov2023direct} shows that the reward model can be represented by the language model,
and aligns the language model through training the reward model without reinforcement learning.
In contrast,
Kahneman-Tversky optimization (KTO) \citep{ethayarajh2024kto} uses prospect theory \citep{tversky1992advances,kahneman2013prospect}, which captures how humans make decisions under uncertainty, to model human preferences.
KTO can be regarded as a generalization of DPO.

Although these approaches are powerful,
human preference models may not accurately capture the real complexities of human preference \citep{wu2022diagnostic},
and consequently these methods lack statistical consistency \citep{vapnik1999overview,mohri2018foundations},
i.e., the guarantee that the language model converges to the true human preference as the number of samples increases \citep{higuchi2025direct}.
To address this,
direct density ratio optimization (DDRO) \citep{higuchi2025direct} presents an alignment approach without assuming any human preference models.
DDRO is based on density ratio estimation,
which directly estimates the ratio between two distributions.
DDRO models the density ratio between the preferred and non-preferred data distributions using the language model,
and optimizes the language model through density ratio estimation.
In contrast to other alignment approaches, DDRO satisfies statistical consistency.
However, it has been empirically observed that training tends to be unstable \citep{higuchi2025direct}.

Our purpose is to address the instability of DDRO and propose an alignment method that is both stable and statistically consistent.
One reason for the instability of DDRO lies in the instability of the density ratio.
Although DDRO models the density ratio between the preferred and non-preferred data distributions,
it may diverge when the supports of these two distributions,
the regions where their values are non-zero,
are different \citep{sugiyama2012density}.
Since preferred and non-preferred responses are likely to have different supports,
this density ratio may diverge, leading to instability in DDRO.
To address this, we employ a more stable approach called relative density ratio estimation \citep{yamada2013relative}.
The relative density ratio is defined as the ratio between the preferred data distribution and a mixture of the preferred and non-preferred data distributions.
Unlike the ordinary density ratio,
the relative density ratio is bounded above and does not diverge,
enabling more stable estimation.
We model the relative density ratio using the language model,
and optimize the language model through relative density ratio estimation.
We theoretically show that our approach is more stable and statistically consistent, and yields significantly tighter convergence guarantees than DDRO.

Our contributions can be summarized as follows:
\begin{itemize}
    \item We propose a novel alignment method called relative density ratio optimization (RDRO), which is based on relative density ratio estimation \citep{yamada2013relative}.
    \item We theoretically show that our RDRO is more stable and statistically consistent, and yields significantly tighter convergence guarantees than DDRO.
    \item We demonstrate the effectiveness of our RDRO with Qwen 2.5 \citep{team2024qwen2} and Llama 3 \citep{grattafiori2024llama}.
\end{itemize}

\section{Preliminaries}

\subsection{Problem Setup}

Let $x$ be the prompt and $y$ be the response.
Given a preferred dataset $\mathcal{D}^{+}=\left(x_{n}^{+},y_{n}^{+}\right)_{n=1}^{N}$
and a non-preferred dataset $\mathcal{D}^{-}=\left(x_{m}^{-},y_{m}^{-}\right)_{m=1}^{M}$,
we assume that the samples are independent and identically distributed within each dataset,
with $\mathcal{D}^{+}$ following $p^{+}(x,y)=p^{+}(y|x)p(x)$ and $\mathcal{D}^{-}$ following $p^{-}(x,y)=p^{-}(y|x)p(x)$.
Note that we assume the same prompt distribution: $p(x)$ for both datasets,
while the response distributions given prompts: $p^{+}(y|x)$ and $p^{-}(y|x)$ differ.
Following \citet{higuchi2025direct}, we use ``true human preference'' to denote the preference data distribution $p^{+}$ defined by an annotator.
Note that this term does not mean a universal preference shared by all humans.
Let $p_{\theta}(y|x)$ be the policy of the language model to be aligned (e.g., a pre-trained or supervised fine-tuning model) with parameter $\theta$,
and $p_{\mathrm{ref}}(y|x)$ be the reference policy,
which serves as a regularizer for $p_{\theta}(y|x)$ and is typically set to a frozen copy of the initial $p_{\theta}(y|x)$.
Our goal is to align $p_{\theta}(y|x)$ with $p^{+}(y|x)$ using $\mathcal{D}^{+}$, $\mathcal{D}^{-}$, and $p_{\mathrm{ref}}(y|x)$.

\subsection{Direct Density Ratio Optimization}
\label{sec:ddro}

Direct density ratio optimization (DDRO) \citep{higuchi2025direct} tries to align the policy $p_{\theta}(y|x)$ with the true preference $p^{+}(y|x)$ using density ratio estimation \citep{sugiyama2012density}.
DDRO assumes that the reference policy $p_{\mathrm{ref}}(y|x)$ can be written as follows:
\begin{align}
p_{\mathrm{ref}}(y|x)=\alpha p^{+}(y|x)+(1-\alpha)p^{-}(y|x),
\label{eq:mixture}
\end{align}
where $\alpha \in (0,1)$ is a hyperparameter.
This assumption can be interpreted as labeling the response of the reference policy $p_{\mathrm{ref}}(y|x)$ as preferred with probability $\alpha$ and as non-preferred with probability $1-\alpha$.
This assumption approximates $p(+|x)$ by a prompt-independent constant $\alpha$,
where $p(+|x)$ denotes the probability that a response sampled from $p_{\mathrm{ref}}(\cdot|x)$ is labeled as preferred \citep{higuchi2025direct}.
We discuss its limitations in Section~\ref{sec:conclusion}.
With this assumption, the following relation holds between the reference policy $p_{\mathrm{ref}}(y|x)$ and the true preference $p^{+}(y|x)$:
\begin{align}
\frac{p_{\mathrm{ref}}(y|x)}{p^{+}(y|x)}
=\frac{\alpha p^{+}(y|x)+(1-\alpha)p^{-}(y|x)}{p^{+}(y|x)}
=\alpha+(1-\alpha)\frac{p^{-}(y|x)}{p^{+}(y|x)}.
\end{align}
Accordingly, the density ratio $p^{-}(y|x)/p^{+}(y|x)$ is written as follows:
\begin{align}
\frac{p^{-}(y|x)}{p^{+}(y|x)}=\frac{1}{1-\alpha}\frac{p_{\mathrm{ref}}(y|x)}{p^{+}(y|x)}-\frac{\alpha}{1-\alpha}\equiv g^{*}(y|x).
\end{align}
By replacing the true preference $p^{+}(y|x)$ with the policy $p_{\theta}(y|x)$,
this density ratio is modeled as follows:
\begin{align}
g_{\theta}(y|x)\equiv\frac{1}{1-\alpha}\frac{p_{\mathrm{ref}}(y|x)}{p_{\theta}(y|x)}-\frac{\alpha}{1-\alpha}.
\end{align}

DDRO tries to align $p_{\theta}(y|x)$ with $p^{+}(y|x)$ by minimizing the Bregman divergence \citep{sugiyama2012density} between $g^{*}(y|x)$ and $g_{\theta}(y|x)$:
\begin{align}
\mathcal{L}_{\mathrm{DRE}}(\theta)=\mathbb{E}_{p^{+}(x,y)}\left[\mathrm{Breg}_{f}\left(g^{*}(y|x)\Vert g_{\theta}(y|x)\right)\right],
\end{align}
where
\begin{align}
\mathrm{Breg}_{f}\left(u\Vert v\right)=f(u)-f(v)-f^{\prime}(v)(u-v),
\end{align}
$\mathbb{E}[\cdot]$ is the expectation,
$f: \mathbb{R} \to \mathbb{R}$ is a differentiable and strictly convex function,
and $f^{\prime}$ is the derivative of $f$.

When $f(t)=t\log t-(1+t)\log(1+t)$,
this optimization problem can be reduced to the logistic regression problem \citep{sugiyama2012density}:
\begin{align}
\mathcal{L}_{\mathrm{DRE}}(\theta)=\mathbb{E}_{p^{+}(x,y)}\left[\log(1+g_{\theta}(y|x))\right]
+\mathbb{E}_{p^{-}(x,y)}\left[\log(1+1/g_{\theta}(y|x))\right].
\end{align}
Using the datasets $\mathcal{D}^{+}$ and $\mathcal{D}^{-}$,
$\mathcal{L}_{\mathrm{DRE}}(\theta)$ is approximated as follows:
\begin{align}
\frac{1}{N}\sum_{n=1}^{N}\log(1+g_{\theta}(y_{n}^{+}|x_{n}^{+}))
+\frac{1}{M}\sum_{m=1}^{M}\log(1+1/g_{\theta}(y_{m}^{-}|x_{m}^{-})).
\end{align}

Unfortunately, this optimization is unstable in practice due to gradient spikes \citep{higuchi2025direct}.
To mitigate this,
DDRO heuristically applies a monotonically increasing function $S(t)=\log\sigma(t)$,
where $\sigma(t)=1/(1+\exp(-t))$ is the sigmoid function,
to this objective as follows:
\begin{align}
\hat{\mathcal{L}}_{\mathrm{DRE}}(\theta)=\frac{1}{N}\sum_{n=1}^{N}S(\log(1+g_{\theta}(y_{n}^{+}|x_{n}^{+})))
+\frac{1}{M}\sum_{m=1}^{M}S(\log(1+1/g_{\theta}(y_{m}^{-}|x_{m}^{-}))).
\label{eq:stabilized_dre}
\end{align}
Since just minimizing this objective will lead to catastrophic forgetting \citep{kirkpatrick2017overcoming},
DDRO introduces the regularizer based on the Kullback-Leibler (KL) divergence between $p_{\theta}(y|x)$ and $p_{\mathrm{ref}}(y|x)$ as follows:
\begin{align}
\mathcal{L}_{\mathrm{DDRO}}(\theta)=\hat{\mathcal{L}}_{\mathrm{DRE}}(\theta)+\beta\mathrm{KL}\left(p_{\theta}\Vert p_{\mathrm{ref}}\right),
\end{align}
where $\beta$ is a hyperparameter that controls the strength of the regularizer.

DDRO has the following property for the optimal parameter $\theta^{*}=\argmin_{\theta}\mathcal{L}_{\mathrm{DRE}}(\theta)$:
\begin{align}
\mathcal{L}_{\mathrm{DRE}}(\theta^{*})=0
\Leftrightarrow g_{\theta^{*}}=g^{*}
\Leftrightarrow p_{\theta^{*}}=p^{+}.
\end{align}
Furthermore,
as the number of samples increases,
the policy $p_{\theta}(y|x)$ converges to the true preference $p^{+}(y|x)$ at the rate of $\mathcal{O}(1/\sqrt{N}+1/\sqrt{M})$.
Although DDRO has this ideal statistical consistency for the original objective $\mathcal{L}_{\mathrm{DRE}}$,
training without the heuristic stabilization leads to gradient explosion.
Conversely, the heuristic transformation $S$ modifies the original objective $\mathcal{L}_{\mathrm{DRE}}$,
causing the statistical consistency guarantee to no longer hold for $\hat{\mathcal{L}}_{\mathrm{DRE}}$ used in practice.
This fundamental tension between stability and statistical consistency motivates our proposed method.

\section{Proposed Method}
\label{sec:rdro}

One reason for the instability of DDRO lies in the instability of the density ratio $g_{\theta}(y|x)$.
Although $g_{\theta}(y|x)$ models the true density ratio $p^{-}(y|x)/p^{+}(y|x)$,
it may diverge when the supports of these two distributions,
the regions where their values are non-zero,
are different.
Since the preferred response $y^{+}$ and the non-preferred response $y^{-}$ are likely to have different supports,
the density ratio may diverge, leading to instability in DDRO.

In this section,
we propose an alignment method that is both stable and statistically consistent by addressing the instability of DDRO.
To this end,
we employ a more stable approach called relative density ratio estimation \citep{yamada2013relative},
which generalizes density ratio estimation.

\subsection{Relative Density Ratio Optimization}

First, we define the relative density ratio $r^{*}(y|x)$ as follows:
\begin{align}
r^{*}(y|x)\equiv\frac{p^{+}(y|x)}{p_{\mathrm{ref}}(y|x)}=\frac{p^{+}(y|x)}{\alpha p^{+}(y|x)+(1-\alpha)p^{-}(y|x)}.
\end{align}
This relative density ratio can be regarded as a generalization of the density ratio since it reduces to the standard density ratio when $\alpha=0$.
By definition, $r^{*}(y|x)\in\left[0,1/\alpha\right]$,
as also illustrated in Figure~\ref{fig:two_ratios} (Appendix~\ref{sec:illustrations}).
Hence, unlike the density ratio,
the relative density ratio does not diverge and can be estimated in a stable manner.

By replacing the true preference $p^{+}(y|x)$ with the policy $p_{\theta}(y|x)$,
we model this relative density ratio as follows:
\begin{align}
r_{\theta}(y|x)\equiv\frac{p_{\theta}(y|x)}{p_{\mathrm{ref}}(y|x)}.
\end{align}
Analogous to DDRO,
we align the policy $p_{\theta}(y|x)$ with the true preference $p^{+}(y|x)$ by minimizing the Bregman divergence between $r^{*}(y|x)$ and $r_{\theta}(y|x)$:
\begin{align}
\mathcal{L}_{\mathrm{RDRE}}(\theta)=\mathbb{E}_{p_{\mathrm{ref}}(x,y)}\left[\mathrm{Breg}_{f}\left(r^{*}(y|x)\Vert r_{\theta}(y|x)\right)\right],
\end{align}
where $p_{\mathrm{ref}}(x,y)=p_{\mathrm{ref}}(y|x)p(x)$.
When $f(t)=t\log t-(1+t)\log(1+t)$,
this optimization problem can be reduced to the logistic regression problem \citep{sugiyama2012density}:
\begin{align}
\mathcal{L}_{\mathrm{RDRE}}(\theta)=\mathbb{E}_{p_{\mathrm{ref}}(x,y)}\left[\log(1+r_{\theta}(y|x))\right]
+\mathbb{E}_{p^{+}(x,y)}\left[\log(1+1/r_{\theta}(y|x))\right].
\end{align}

Defining $T_{\theta}(x,y)=\log p_{\theta}(y|x)-\log p_{\mathrm{ref}}(y|x)$,
we have $r_{\theta}(y|x)=\exp(T_{\theta}(x,y))$.
Substituting this into $\mathcal{L}_{\mathrm{RDRE}}(\theta)$, we obtain:
\begin{align}
\mathcal{L}_{\mathrm{RDRE}}(\theta)
&=\mathbb{E}_{p_{\mathrm{ref}}(x,y)}\left[\log(1+\exp(T_{\theta}(x,y)))\right]
+\mathbb{E}_{p^{+}(x,y)}\left[\log(1+1/\exp(T_{\theta}(x,y)))\right]\nonumber\\
&=\mathbb{E}_{p_{\mathrm{ref}}(x,y)}\left[\varphi(T_{\theta}(x,y))\right]
+\mathbb{E}_{p^{+}(x,y)}\left[\varphi(T_{\theta}(x,y))-T_{\theta}(x,y)\right]\nonumber\\
&=\mathbb{E}_{p^{+}(x,y)}\left[(1+\alpha)\varphi(T_{\theta}(x,y))-T_{\theta}(x,y)\right]
+\mathbb{E}_{p^{-}(x,y)}\left[(1-\alpha)\varphi(T_{\theta}(x,y))\right],
\end{align}
where $\varphi(t)=\log(1+\exp(t))$ is the softplus function.
To obtain the last expression,
we used $p_{\mathrm{ref}}(x,y)=\alpha p^{+}(x,y)+(1-\alpha)p^{-}(x,y)$.
Figure~\ref{fig:loss_shape} (Appendix~\ref{sec:illustrations}) shows our loss functions for preferred and non-preferred samples.
Note that $S(t)$ used in Eq.~(\ref{eq:stabilized_dre}) can be written as $S(t)=-\varphi(-t)$.
This shows that the stabilizing transformation employed in DDRO arises naturally in our approach.

Using the datasets $\mathcal{D}^{+}$ and $\mathcal{D}^{-}$, $\mathcal{L}_{\mathrm{RDRE}}(\theta)$ is approximated as follows:
\begin{align}
\mathcal{L}_{\mathrm{RDRO}}(\theta)\equiv
\frac{1}{N}\sum_{n=1}^{N}\left\{ (1+\alpha)\varphi(T_{\theta}(x_{n}^{+},y_{n}^{+}))-T_{\theta}(x_{n}^{+},y_{n}^{+})\right\}
+\frac{1}{M}\sum_{m=1}^{M}(1-\alpha)\varphi(T_{\theta}(x_{m}^{-},y_{m}^{-})).
\label{eq:proposed}
\end{align}
We optimize this objective using a stochastic optimization method such as AdamW \citep{loshchilov2017decoupled}.
As we will show later, we can omit the KL regularizer used in DDRO,
since $\alpha$ serves as a regularization parameter.
We refer to our approach as relative density ratio optimization (RDRO).
The PyTorch implementation and pseudocode are provided in Appendices~\ref{sec:implementation} and \ref{sec:pseudocode}, respectively.

In the following section,
we show that our approach is both stable and statistically consistent.

\subsection{Stability Analysis}

First, we analyze the stability of the proposed method by examining the gradient of the loss function.
Since
\begin{align}
\nabla_{\theta}\varphi(T_{\theta}(x,y))=\sigma(T_{\theta}(x,y))\nabla_{\theta}T_{\theta}(x,y)
\end{align}
and
\begin{align}
\nabla_{\theta}T_{\theta}(x,y)=\nabla_{\theta}\log p_{\theta}(y|x),
\end{align}
the gradient of $\mathcal{L}_{\mathrm{RDRE}}(\theta)$ can be written as follows:
\begin{align}
\nabla_{\theta}\mathcal{L}_{\mathrm{RDRE}}(\theta)=\mathbb{E}_{p^{+}(x,y)}\left[c_{\theta}^{+}(x,y)\nabla_{\theta}\log p_{\theta}(y|x)\right]
+\mathbb{E}_{p^{-}(x,y)}\left[c_{\theta}^{-}(x,y)\nabla_{\theta}\log p_{\theta}(y|x)\right],
\end{align}
where
\begin{align}
c_{\theta}^{+}(x,y)=(1+\alpha)\sigma(T_{\theta}(x,y))-1,
\qquad c_{\theta}^{-}(x,y)=(1-\alpha)\sigma(T_{\theta}(x,y)).
\end{align}

For the preferred dataset $\mathcal{D}^{+}$, we would like $p_{\theta}(y|x)$ to be large,
whereas for the non-preferred dataset $\mathcal{D}^{-}$, we would like $p_{\theta}(y|x)$ to be small.
Thus, $c_{\theta}^{+}(x,y)\leq 0$ and $c_{\theta}^{-}(x,y)\geq 0$ are required.
By definition, $c_{\theta}^{-}(x,y)$ is always non-negative,
and when $p_{\theta}(y|x)$ becomes sufficiently small for the non-preferred data,
$c_{\theta}^{-}(x,y)$ becomes zero.

On the other hand, for $c_{\theta}^{+}(x,y)$, the following inequality holds:
\begin{align}
c_{\theta}^{+}(x,y)\leq 0\Leftrightarrow r_{\theta}(y|x) \equiv \frac{p_{\theta}(y|x)}{p_{\mathrm{ref}}(y|x)}\leq\frac{1}{\alpha}.
\end{align}
For the preferred data $(x^{+},y^{+})$,
the gradient is negative when $r_{\theta}(y^{+}|x^{+})\leq1/\alpha$.
This enlarges $p_{\theta}(y^{+}|x^{+})$ until $c_{\theta}^{+}(x^{+},y^{+})$ becomes zero at $r_{\theta}(y^{+}|x^{+})=1/\alpha$.
Hence, the constraint on the relative density ratio,
$r_{\theta}(y^{+}|x^{+})\in[0,1/\alpha]$,
is automatically satisfied.

From another perspective, $p_{\theta}(y|x)$ is trained so that it does not exceed $p_{\mathrm{ref}}(y|x)/\alpha$.
When $\alpha$ is large, $p_{\theta}(y|x)$ remains close to $p_{\mathrm{ref}}(y|x)$,
whereas when $\alpha$ is small,
$p_{\theta}(y|x)$ can deviate from $p_{\mathrm{ref}}(y|x)$.
Hence, $\alpha$ can be interpreted as a regularization parameter.

Note that this bounding mechanism differs from gradient clipping.
Gradient clipping rescales the aggregated gradient only after its direction has been determined.
Thus, an unbounded density ratio may allow a few samples to dominate the aggregated gradient
and produce a poorly conditioned update direction that clipping cannot correct.
In contrast, RDRO bounds the coefficient for each sample before aggregating the gradients,
since $c_{\theta}^{+}\in[-1,\alpha]$ and $c_{\theta}^{-}\in[0,1-\alpha]$.
Gradient clipping remains useful for controlling large gradients of $\log p_{\theta}(y|x)$.

\subsection{Statistical Consistency Analysis}

Next, we analyze the statistical consistency of the proposed method.
For the optimal parameter $\theta^{*}$ that minimizes the true risk $\mathcal{L}_{\mathrm{RDRE}}(\theta)=\mathbb{E}_{p_{\mathrm{ref}}}\left[\mathrm{Breg}_{f}\left(r^{*}\Vert r_{\theta}\right)\right]$,
since $f$ is strictly convex,
the following consistency holds:
\begin{align}
\mathcal{L}_{\mathrm{RDRE}}(\theta^{*})=0
\Leftrightarrow r_{\theta^{*}}=r^{*}
\Leftrightarrow p_{\theta^{*}}=p^{+}.
\end{align}
Although this is an ideal property,
since the number of samples is finite in practice,
it is necessary to evaluate the convergence rate.

As a preliminary step,
we evaluate the Lipschitz constant of the Bregman divergence,
which can be rewritten as follows:
\begin{align}
\mathrm{Breg}_{f}\left(u\Vert v\right)=\underbrace{-f(v)+f^{\prime}(v)v}_{\psi_{1}(v)}+\underbrace{(-f^{\prime}(v))}_{\psi_{2}(v)}u+\underbrace{f(u)}_{Const.}.
\end{align}
Let $L_{1}$ and $L_{2}$ be the Lipschitz constants of $\psi_{1}$ and $\psi_{2}$, respectively.
Then, the Lipschitz constant of the Bregman divergence can be written as follows:
\begin{align}
\mathrm{Lip}\left(\mathrm{Breg}_{f}\left(u\Vert\cdot\right)\right)=L_{1}+(\sup\left|u\right|)L_{2}.
\end{align}
For the proposed method, since $r^{*} \leq 1/\alpha$, this Lipschitz constant becomes:
\begin{align}
C_{\mathrm{Lip}}
=\mathrm{Lip}\left(\mathrm{Breg}_{f}\left(r^{*}\Vert\cdot\right)\right)
=L_{1}+(\sup\left|r^{*}\right|)L_{2}
=L_{1}+\frac{1}{\alpha}L_{2}.
\end{align}

Let $\Theta$ be the parameter set,
and $\mathcal{H}=\left\{ p_{\theta}\mid\theta\in\Theta\right\}$ be the corresponding model class,
following the convention of \citep{higuchi2025direct}.
For this $\mathcal{H}$,
the Rademacher complexity \citep{mohri2018foundations} is defined as $\mathcal{R}_{K}(\mathcal{H})$,
where $K$ is the number of samples.
The Rademacher complexity measures the richness of a model class and is commonly used to bound generalization error.
Based on this, we evaluate the estimation error of the policy with the true preference $p^{+}$.

\begin{theorem}
We define $\hat{\theta}=\argmin_{\theta}\hat{\mathcal{L}}_{\mathrm{RDRE}}(\theta)$,
where $\hat{\mathcal{L}}_{\mathrm{RDRE}}(\theta)$ is the empirical approximation of $\mathcal{L}_{\mathrm{RDRE}}(\theta)$ by using the datasets $\mathcal{D}^{\pm}=\left\{ \mathcal{D}^{+},\mathcal{D}^{-}\right\}$.
Then, the estimation error of $p_{\hat{\theta}}$ with $p^{+}$ is bounded above as follows:
\begin{align}
\mathbb{E}_{\mathcal{D}^{\pm}}\left[\mathbb{E}_{p^{+}}\left[\left(p_{\hat{\theta}}-p^{+}\right)^{2}\right]\right]\leq\frac{2}{\alpha\mu}\left[\inf_{\theta\in\Theta}\mathcal{L}_{\mathrm{RDRE}}(\theta)
+4C_{\mathrm{Lip}}\left\{ \alpha\mathcal{R}_{N}(\mathcal{H})+(1-\alpha)\mathcal{R}_{M}(\mathcal{H})\right\} \right],
\label{eq:consistency}
\end{align}
where $\mu$ satisfies $f^{\prime\prime}(t)\geq\mu$ for all $t$.
\label{thm:consistency}
\end{theorem}

\begin{proof}[Proof sketch]
The key insight is the boundedness $r^{*}\leq 1/\alpha$, which (i) yields a smaller Lipschitz constant $C_{\mathrm{Lip}} = L_1 + L_2/\alpha$ than that of DDRO, and (ii) gives $\mathbb{E}_{p^{+}}[(p_{\theta}-p^{+})^{2}] \leq \frac{2}{\alpha\mu}\mathcal{L}_{\mathrm{RDRE}}(\theta)$ via $p_{\theta}-p^{+}=p_{\mathrm{ref}}(r_{\theta}-r^{*})$.
Combining these with a standard Rademacher complexity argument gives the result; see Appendix~\ref{sec:proof} for the complete proof.
\end{proof}

Since $\mathcal{R}_{K}(\mathcal{H})=\mathcal{O}(1/\sqrt{K})$ usually holds \citep{mohri2018foundations},
the right-hand side of Eq.~(\ref{eq:consistency}) can be evaluated as
\begin{align}
\mathcal{O}\left(\inf_{\theta\in\Theta}\mathcal{L}_{\mathrm{RDRE}}(\theta)+1/\sqrt{N}+1/\sqrt{M}\right).
\end{align}
In addition,
if the model class $\mathcal{H}$ contains the true preference $p^{+}$,
$\inf_{\theta\in\Theta}\mathcal{L}_{\mathrm{RDRE}}(\theta)$ becomes zero.
Therefore, $\mathbb{E}_{p^{+}}\left[\left(p_{\hat{\theta}}-p^{+}\right)^{2}\right]\to 0$ as $N,M\to \infty$.

\subsection{Comparison of Convergence Rates}
\label{sec:order}

Finally, we compare the convergence rates of the proposed method and DDRO.
According to \citep{higuchi2025direct},
the upper bound of the estimation error for DDRO is:
\begin{align}
\mathbb{E}_{\mathcal{D}^{\pm}}\left[\mathbb{E}_{p^{+}}\left[\left(p_{\hat{\theta}}-p^{+}\right)^{2}\right]\right]\leq\frac{2\left(1-\alpha\right)^{2}}{\alpha^{2}m_{+}^{2}\mu}
\left[\inf_{\theta\in\Theta}\mathcal{L}_{\mathrm{DRE}}(\theta)+4C_{\mathrm{Lip}}^{\prime}\left\{ \mathcal{R}_{N}(\mathcal{H})+\mathcal{R}_{M}(\mathcal{H})\right\} \right],
\end{align}
where
\begin{align}
m_{+}=\min_{x\in\mathcal{X},y\in\mathrm{supp}(p^{+}(\cdot|x))}p^{+}(y|x)>0
\end{align}
is the minimum value of $p^{+}(y|x)$ within its support $\mathrm{supp}(p^{+}(\cdot|x))$,
$\mathcal{X}$ is the input space,
and
\begin{align}
C_{\mathrm{Lip}}^{\prime}
=\mathrm{Lip}\left(\mathrm{Breg}_{f}\left(g^{*}\Vert\cdot\right)\right)
=L_{1}+(\sup\left|g^{*}\right|)L_{2}.
\end{align}

When $\mathcal{H}$ contains $p^{+}$,
both $\inf_{\theta\in\Theta}\mathcal{L}_{\mathrm{DRE}}(\theta)$ and $\inf_{\theta\in\Theta}\mathcal{L}_{\mathrm{RDRE}}(\theta)$ are zero.
We therefore compare the remaining upper-bound terms.
For $m_{+}\ll1$, the leading coefficients in the two bounds satisfy:
\begin{align}
\frac{2}{\alpha\mu}\ll\frac{2\left(1-\alpha\right)^{2}}{\alpha^{2}m_{+}^{2}\mu}.
\label{eq:m_plus}
\end{align}
More precisely, this inequality holds whenever $\alpha<\{(\sqrt{m_{+}^{2}+4}-m_{+})/2\}^{2}=1-m_{+}+O(m_{+}^{2})$.
Hence, for any fixed $\alpha\in(0,1)$, Eq.~(\ref{eq:m_plus}) holds for sufficiently small $m_{+}$.
Appendix~\ref{sec:order_details} gives the full derivation.

Moreover, since $1/\alpha\ll\sup\left|g^{*}\right|$, the following inequality holds:
\begin{align}
C_{\mathrm{Lip}}=L_{1}+\frac{1}{\alpha}L_{2}\ll L_{1}+(\sup\left|g^{*}\right|)L_{2}=C_{\mathrm{Lip}}^{\prime}.
\end{align}

In addition, since $\alpha \in (0,1)$, the following inequality holds:
\begin{align}
\alpha\mathcal{R}_{N}(\mathcal{H})+(1-\alpha)\mathcal{R}_{M}(\mathcal{H})\leq\mathcal{R}_{N}(\mathcal{H})+\mathcal{R}_{M}(\mathcal{H}).
\end{align}

From the above discussion,
although both the proposed method and DDRO share the order of $\mathcal{O}(1/\sqrt{N}+1/\sqrt{M})$,
our coefficient terms are significantly smaller than those for DDRO.
In particular, in DDRO,
the occurrence of $1/m_{+}$ and the potentially large value of $\sup\left|g^{*}\right|$ arise from the instability of the density ratio $g^{*}\in\left[0,\infty\right)$.
The proposed method avoids this order deterioration by using the relative density ratio $r^{*}\in\left[0,1/\alpha\right]$ instead of the density ratio $g^{*}$.

To summarize our contributions, our approach is more stable and statistically consistent,
and yields significantly tighter convergence guarantees than DDRO.

\section{Related Work}

Many alignment methods assume specific human preference models,
such as the Bradley-Terry model \citep{bradley1952rank} and prospect theory \citep{tversky1992advances,kahneman2013prospect}.
The Bradley-Terry model is widely used in reinforcement learning approaches
such as RLHF \citep{ouyang2022training,uesato2022solving,ahmadian2024back},
as well as in DPO and its variants \citep{rafailov2023direct,xu2024contrastive,xie2024exploratory,meng2024simpo}.
KTO \citep{ethayarajh2024kto} is based on prospect theory and can be regarded as a generalization of DPO.
Unlike RLHF and DPO, which require a preferred and a non-preferred response for each prompt,
KTO can be applied to individual responses with binary labels indicating whether they are preferred.
However, specific human preference models such as the Bradley-Terry model may fail to accurately capture true human preferences \citep{wu2022diagnostic}.
Cyclic preference is a canonical example of this limitation \citep{higuchi2025direct};
Appendix~\ref{sec:cyclic_preference} gives an explicit construction and proof.

To address this limitation,
several approaches have been presented,
including those based on energy-based models \citep{hong2024energy},
Nash equilibrium \citep{munos2023nash},
binary classification \citep{jung2024binary},
and density ratio estimation \citep{higuchi2025direct,kim2025preference}.
In particular, DDRO \citep{higuchi2025direct} has statistical consistency as shown in Section~\ref{sec:ddro}.
However, DDRO is also known to suffer from training instability.
One cause lies in the instability of the density ratio,
which may diverge when the supports of the two distributions differ \citep{sugiyama2012density}.
Our approach avoids this instability by employing the bounded relative density ratio \citep{yamada2013relative}.
Consequently, our approach offers both training stability and statistical consistency,
and yields significantly tighter convergence guarantees than DDRO.
Note that \citet{kim2025preference} is also based on density ratio estimation,
but requires pairwise preference data and does not provide a guarantee of statistical consistency.

Note that using pointwise supervision does not by itself avoid the limitations of specific human preference models.
For example, SPIN \citep{chen2024self} constructs pairwise comparisons between supervised responses and responses from the current model,
and iteratively trains a discriminator through self-play.
Although SPIN does not explicitly assume the Bradley-Terry model,
its final objective coincides with that of DPO and therefore cannot represent cyclic preferences.
In contrast, our approach uses only pointwise labels without assuming any human preference models and satisfies statistical consistency.

\section{Experiments}
\label{sec:experiments}

We primarily compared RDRO against KTO \citep{ethayarajh2024kto} and DDRO \citep{higuchi2025direct}, which also learn from pointwise labels.
Where pairwise labels are available, we additionally report DPO \citep{rafailov2023direct} and SimPO \citep{meng2024simpo} to provide a broader comparison.
We exclude these methods from the primary comparison because they require stronger pairwise supervision.
Due to space constraints, additional results are deferred to Appendices~\ref{sec:bbh} and \ref{sec:training_dynamics}.

\subsection{Base Models, Preference Datasets, and Metrics}

We applied these alignment methods to the following base models:
Qwen 2.5 Instruct (\href{https://huggingface.co/Qwen/Qwen2.5-1.5B-Instruct}{1.5B} and \href{https://huggingface.co/Qwen/Qwen2.5-3B-Instruct}{3B}) \citep{team2024qwen2},
and Llama 3 Instruct (\href{https://huggingface.co/meta-llama/Llama-3.2-3B-Instruct}{3B} and \href{https://huggingface.co/meta-llama/Llama-3.1-8B-Instruct}{8B}) \citep{grattafiori2024llama}.
We refer to them as Qwen-1.5B, Qwen-3B, Llama-3B, and Llama-8B, respectively.

As the preference datasets, we used \href{https://huggingface.co/datasets/trl-lib/ultrafeedback-gpt-3.5-turbo-helpfulness}{UltraFeedback GPT-3.5-Turbo Helpfulness} (UF-G)
and \href{https://huggingface.co/datasets/trl-lib/kto-mix-14k}{KTO Mix 14K} (MIX-14K).
UF-G is derived from the \href{https://huggingface.co/datasets/openbmb/UltraFeedback}{UltraFeedback} \citep{cui2023ultrafeedback} and filtered for helpfulness,
containing preferred and non-preferred responses to various prompts labeled by GPT-3.5-turbo.
MIX-14K mixes the following three datasets:
\href{https://huggingface.co/datasets/argilla/distilabel-capybara-dpo-7k-binarized}{Capybara},
\href{https://huggingface.co/datasets/argilla/distilabel-intel-orca-dpo-pairs}{Orca},
and \href{https://huggingface.co/datasets/argilla/ultrafeedback-binarized-preferences-cleaned}{UltraFeedback} by sampling from each source in equal proportions.
UF-G contains 6,088 preferred and 9,644 non-preferred samples,
while MIX-14K contains 6,750 preferred and 6,750 non-preferred samples.

As the evaluation metric, we used \href{https://github.com/tatsu-lab/alpaca_eval}{AlpacaEval length-controlled (LC) win rates} \citep{dubois2024length}.
We first generated responses from the aligned models on the \href{https://huggingface.co/datasets/tatsu-lab/alpaca_eval}{alpaca\_eval} prompts,
then computed win rates against a baseline using an LLM-based auto-annotator,
and finally reported LC win rates that regress out AlpacaEval’s bias toward longer outputs.
We used \texttt{gpt-4-turbo-2024-04-09} for the baseline and \texttt{weighted\_alpaca\_eval\_gpt4\_turbo} for the auto-annotator.

\subsection{Setup}

We implemented all methods using the \href{https://github.com/huggingface/trl}{Hugging Face TRL library}.
We used the KTO implementation in TRL and implemented DDRO and the proposed method by extending this KTO implementation.
For training, we used AdamW \citep{loshchilov2017decoupled} with a cosine learning rate scheduler \citep{loshchilov2016sgdr}.
Following the \href{https://huggingface.co/docs/trl/kto_trainer}{TRL documentation},
we trained all models for one epoch with a learning rate of $5\times10^{-7}$ and a batch size of 128.
We set the maximum gradient norm for gradient clipping to 1.0 and used a warmup ratio of 10\% of the total training steps.

For DDRO and the proposed method,
we set the hyperparameter $\alpha$ to the fraction of preferred samples in each dataset:
0.39 for UF-G and 0.5 for MIX-14K.
This choice is motivated by Eq.~(\ref{eq:mixture}):
under the mixture model, $\alpha$ represents the proportion of preferred responses in the reference policy's output distribution,
so setting it to the empirical fraction of preferred samples in the dataset is natural.
For DDRO, we set the KL regularization coefficient $\beta$ to $0.1$, the default value in the KTO implementation of TRL.
Following \citep{higuchi2025direct}, which explicitly excludes the KL regularization term from gradient computation for training stability,
$\beta$ does not affect the gradient update in practice.

The experiments were conducted on machines equipped with eight NVIDIA H100 GPUs (80GB HBM3),
Intel Xeon Platinum 8480+ CPUs, and 2.0TB of memory.
We used \href{https://github.com/huggingface/accelerate}{Hugging Face Accelerate} for distributed training and DeepSpeed ZeRO-3 \citep{rajbhandari2020zero} for memory efficiency.
We repeated each experiment three times with different random seeds.

\subsection{Results}

Tables \ref{tab:winrates_UF-G} and \ref{tab:winrates_MIX-14K} show the AlpacaEval LC win rates on UF-G and MIX-14K, respectively.
We used boldface to indicate the best result and any results not significantly different from it according to a pairwise t-test with a significance level of 5\%.

\begin{table}[tb]
\caption{
  Comparison of AlpacaEval LC Win Rates on UF-G.
}
\begin{center}
\begin{tabular}{l|rrrr}
\toprule
 & Qwen-1.5B & Qwen-3B & Llama-3B & Llama-8B \\
\midrule
KTO  & \textbf{5.220 $\pm$ 0.104} & 7.920 $\pm$ 0.082 & 9.937 $\pm$ 0.024 & 9.344 $\pm$ 0.296 \\
DDRO & \textbf{5.364 $\pm$ 0.162} & \textbf{7.888 $\pm$ 0.088} & 12.069 $\pm$ 0.113 & 15.660 $\pm$ 0.343 \\
RDRO & \textbf{5.599 $\pm$ 0.122} & \textbf{8.933 $\pm$ 0.318} & \textbf{12.634 $\pm$ 0.101} & \textbf{17.314 $\pm$ 0.078} \\
\bottomrule
\end{tabular}
\label{tab:winrates_UF-G}
\end{center}
\end{table}

\begin{table}[tb]
\caption{
  Comparison of AlpacaEval LC Win Rates on MIX-14K.
  DPO \citep{rafailov2023direct} and SimPO \citep{meng2024simpo} use pairwise labels (DPO-Mix-7K),
  while KTO, DDRO, and RDRO use pointwise labels (MIX-14K).
}
\begin{center}
\begin{tabular}{l|rrrr}
\toprule
 & Qwen-1.5B & Qwen-3B & Llama-3B & Llama-8B \\
\midrule
KTO   & 5.320 $\pm$ 0.173 & 8.240 $\pm$ 0.124 & \textbf{12.763 $\pm$ 0.190} & 13.921 $\pm$ 0.041 \\
DDRO  & 5.169 $\pm$ 0.166 & 8.458 $\pm$ 0.138 & \textbf{12.487 $\pm$ 0.156} & \textbf{17.269 $\pm$ 0.675} \\
RDRO  & 5.297 $\pm$ 0.102 & 8.793 $\pm$ 0.238 & \textbf{12.720 $\pm$ 0.205} & \textbf{17.102 $\pm$ 0.305} \\
\midrule
DPO   & \textbf{7.134 $\pm$ 0.145} & 12.418 $\pm$ 0.085 & 10.056 $\pm$ 0.088 & \textbf{16.039 $\pm$ 0.448} \\
SimPO & 6.557 $\pm$ 0.059 & \textbf{12.612 $\pm$ 0.142} & 9.815 $\pm$ 0.231 & \textbf{15.862 $\pm$ 0.242} \\
\bottomrule
\end{tabular}
\label{tab:winrates_MIX-14K}
\end{center}
\end{table}

We first focus on the results on UF-G.
Table~\ref{tab:winrates_UF-G} shows that,
across all models, the proposed method achieves performance comparable to or better than both KTO and DDRO.
In particular, for Llama-3B and Llama-8B,
the proposed method yields a statistically significant improvement over KTO and DDRO.
We next focus on the results on MIX-14K.
Table~\ref{tab:winrates_MIX-14K} shows that, across all models, the proposed method achieves comparable performance to KTO and DDRO.

This difference across datasets may stem from structural differences between UF-G and MIX-14K.
MIX-14K decomposes each triplet $(x,y_w,y_l)$ from \href{https://huggingface.co/datasets/argilla/dpo-mix-7k}{DPO-Mix-7K} \citep{rafailov2023direct}
into two prompt and response pairs, $(x,y_w)$ and $(x,y_l)$.
Thus, both a preferred and a non-preferred response are available for every prompt,
which may make alignment easier and explain the competitive performance of all methods.
In contrast, UF-G provides only one of $y_w$ and $y_l$ for each prompt,
which may make alignment more challenging and produce larger performance differences.
In practice, it is difficult to always provide both preferred and non-preferred responses for each prompt.
Therefore, achieving strong performance on realistic datasets like UF-G is particularly important.

Table~\ref{tab:winrates_MIX-14K} also reports results for DPO and SimPO trained on DPO-Mix-7K, the pairwise version of MIX-14K.
Note that this comparison is not fully comparable since DPO and SimPO have access to stronger (pairwise) supervision.
Nevertheless, RDRO achieves performance comparable to or better than DPO and SimPO on larger models (Llama-3B and Llama-8B) using only pointwise labels.

\begin{figure}[t]
  \centering
  \begin{subfigure}{0.245\textwidth}
    \includegraphics[width=\linewidth]{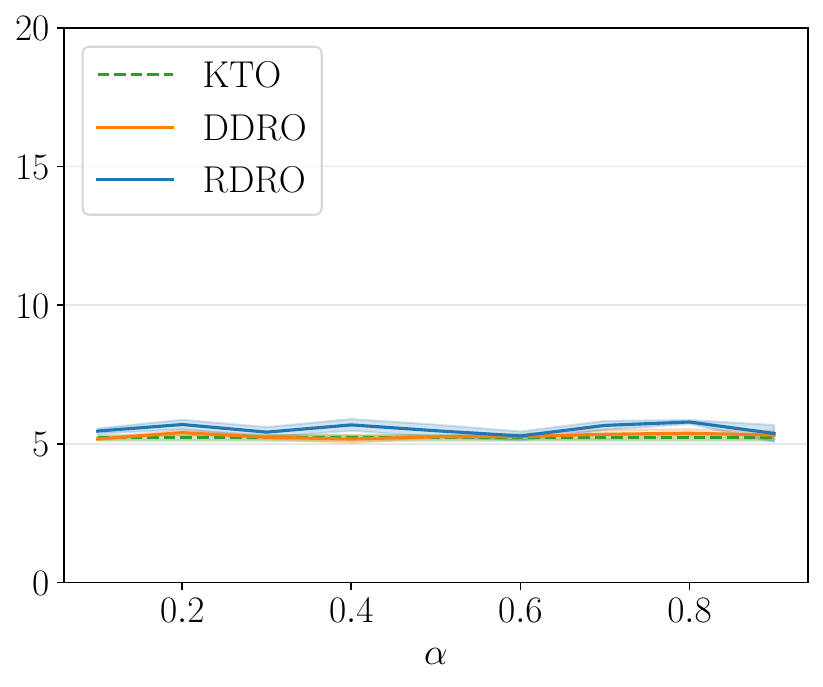}
    \caption{Qwen-1.5B.}
  \end{subfigure}
  \begin{subfigure}{0.245\textwidth}
    \includegraphics[width=\linewidth]{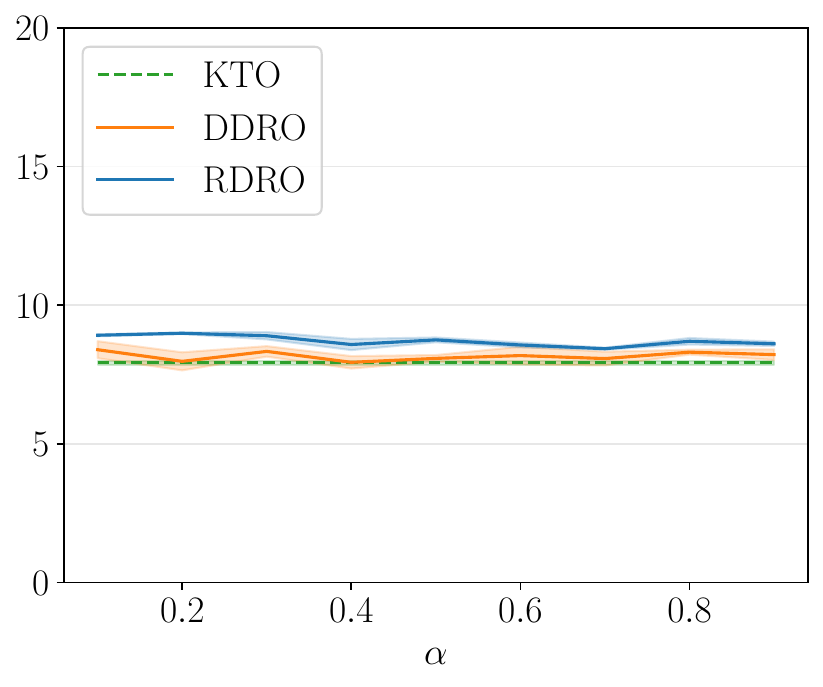}
    \caption{Qwen-3B.}
  \end{subfigure}
  \begin{subfigure}{0.245\textwidth}
    \includegraphics[width=\linewidth]{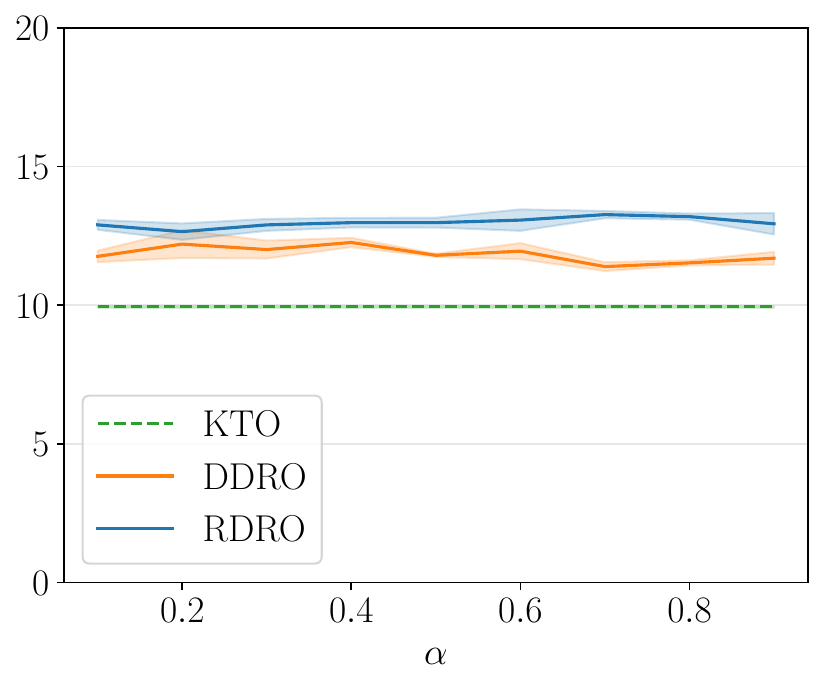}
    \caption{Llama-3B.}
  \end{subfigure}
  \begin{subfigure}{0.245\textwidth}
    \includegraphics[width=\linewidth]{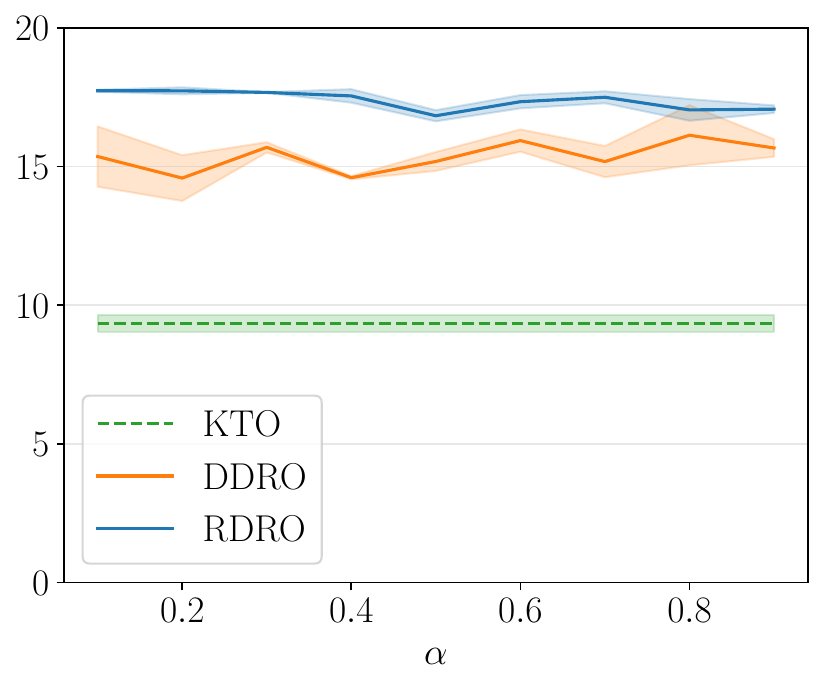}
    \caption{Llama-8B.}
  \end{subfigure}
  \caption{
    AlpacaEval LC win rates across $\alpha$ on UF-G, with shading indicating standard deviations.
  }
  \label{fig:alpha_UF-G}
\end{figure}

\begin{figure}[t]
  \centering
  \begin{subfigure}{0.245\textwidth}
    \includegraphics[width=\linewidth]{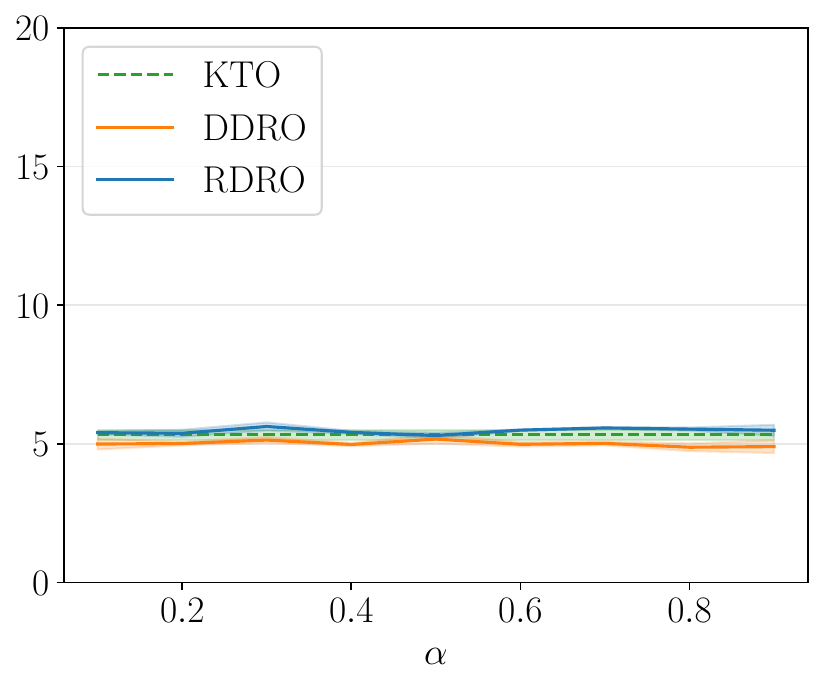}
    \caption{Qwen-1.5B.}
  \end{subfigure}
  \begin{subfigure}{0.245\textwidth}
    \includegraphics[width=\linewidth]{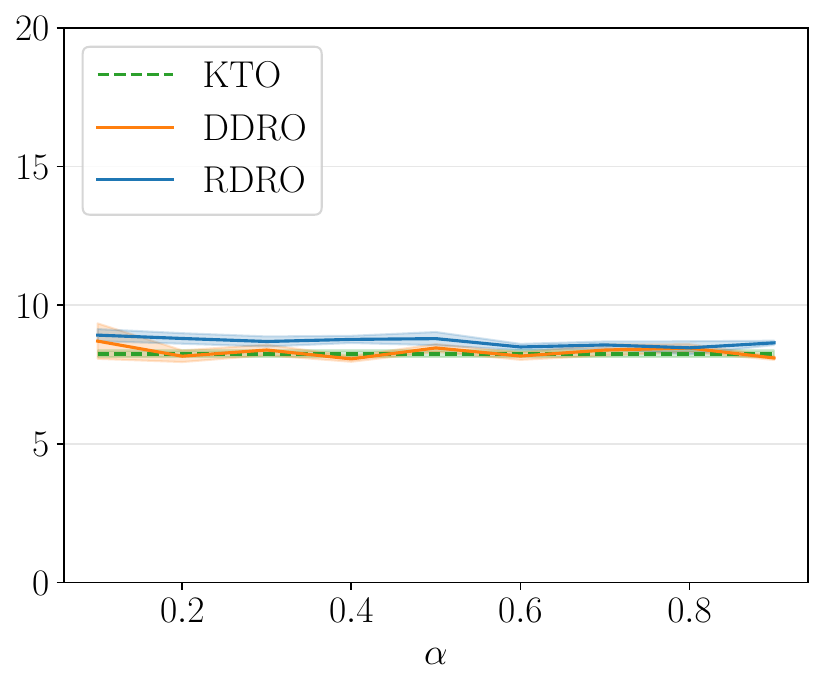}
    \caption{Qwen-3B.}
  \end{subfigure}
  \begin{subfigure}{0.245\textwidth}
    \includegraphics[width=\linewidth]{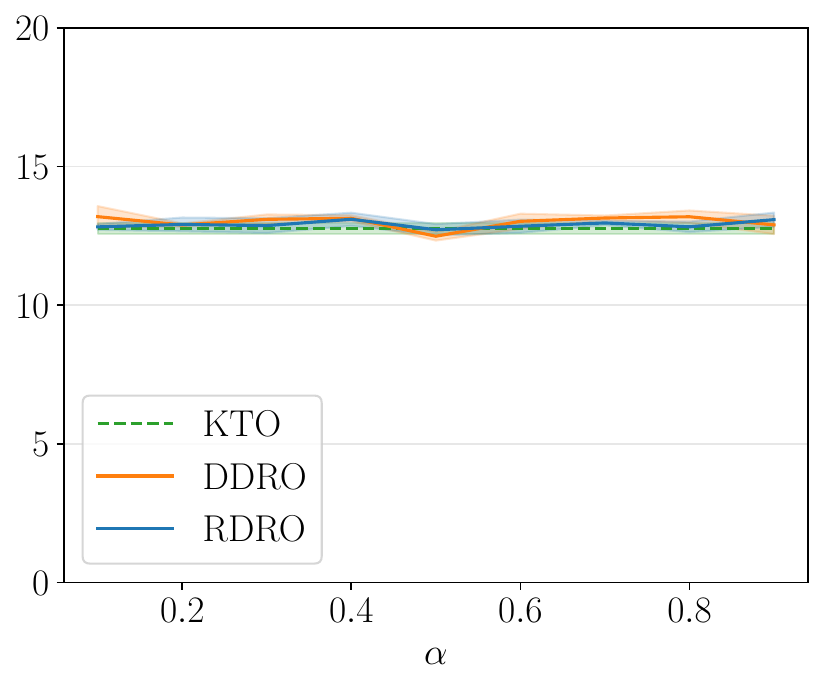}
    \caption{Llama-3B.}
  \end{subfigure}
  \begin{subfigure}{0.245\textwidth}
    \includegraphics[width=\linewidth]{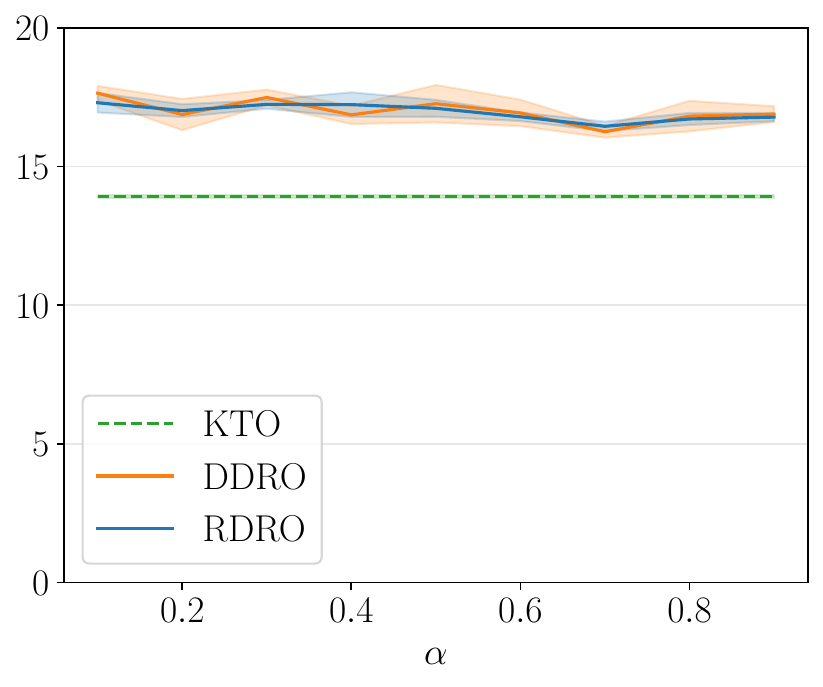}
    \caption{Llama-8B.}
  \end{subfigure}
  \caption{
    AlpacaEval LC win rates across $\alpha$ on MIX-14K, with shading indicating standard deviations.
  }
  \label{fig:alpha_MIX-14K}
\end{figure}

Figures~\ref{fig:alpha_UF-G} and~\ref{fig:alpha_MIX-14K} show AlpacaEval LC win rates across $\alpha$ on UF-G and MIX-14K, respectively.
Across both datasets, RDRO is robust to $\alpha$:
it outperforms DDRO at every tested value on UF-G and remains comparable on MIX-14K.
These results show that RDRO does not require precise tuning of $\alpha$ across the evaluated base models and datasets.

Additional results on BIG-Bench Hard \citep{suzgun2023challenging} are provided in Appendix~\ref{sec:bbh}.
Appendix~\ref{sec:training_dynamics} analyzes the training dynamics to empirically evaluate the stability of the proposed method.
These results demonstrate the strong alignment performance and training stability of our approach.

\section{Conclusion, Limitations, and Future Work}
\label{sec:conclusion}

In this paper, we propose relative density ratio optimization (RDRO), a novel alignment method that resolves the inherent trade-off between training stability and statistical consistency in DDRO.
Experiments across various models and datasets demonstrate that RDRO achieves performance comparable to or better than KTO and DDRO, and is competitive with pairwise methods such as DPO and SimPO despite using only pointwise labels.
One limitation is that our approach requires fully labeled data; incorporating unlabeled data via positive-unlabeled learning \citep{kato2019learning,kato2021non} is a promising direction.
A further limitation concerns the mixture model assumption described in Section~\ref{sec:ddro},
which approximates $p(+|x)$ by a constant $\alpha$ across prompts.
This assumption may be violated when $p(+|x)$ varies across prompts, as can occur in UF-G and MIX-14K;
nevertheless, RDRO performs well on both datasets.
Figures~\ref{fig:alpha_UF-G} and~\ref{fig:alpha_MIX-14K} further show that misspecifying $\alpha$ does not substantially degrade performance,
and the bound in Theorem~\ref{thm:consistency} remains finite for any fixed $\alpha\in(0,1)$.
Replacing $\alpha$ with a prompt-dependent $\alpha(x)$ estimated by a lightweight model \citep{devlin2019bert} would be a natural extension.

\bibliographystyle{plainnat}
\bibliography{main}

\newpage
\appendix

\section{PyTorch Implementation}
\label{sec:implementation}

RDRO is easy to implement in practice;
PyTorch code for the RDRO loss is provided below:
{\small
\begin{verbatim}
import torch
from torch.nn.functional import softplus as sp


def rdro_loss(
    policy_chosen_logps,
    policy_rejected_logps,
    ref_chosen_logps,
    ref_rejected_logps,
    alpha,
):
    """
    policy_chosen_logps: policy logprobs for chosen responses
    policy_rejected_logps: policy logprobs for rejected responses
    ref_chosen_logps: reference policy logprobs for chosen responses
    ref_rejected_logps: reference policy logprobs for rejected responses
    alpha: hyperparameter in (0, 1)
    """

    chosen_logratios = policy_chosen_logps - ref_chosen_logps
    chosen_losses = (1 + alpha) * sp(chosen_logratios) - chosen_logratios

    rejected_logratios = policy_rejected_logps - ref_rejected_logps
    rejected_losses = (1 - alpha) * sp(rejected_logratios)

    return torch.cat([chosen_losses, rejected_losses], dim=0)
\end{verbatim}
}

\section{Illustrations of Density Ratios and Loss Functions}
\label{sec:illustrations}

\begin{figure}[h]
  \centering
  \begin{subfigure}{0.495\textwidth}
    \includegraphics[width=\linewidth]{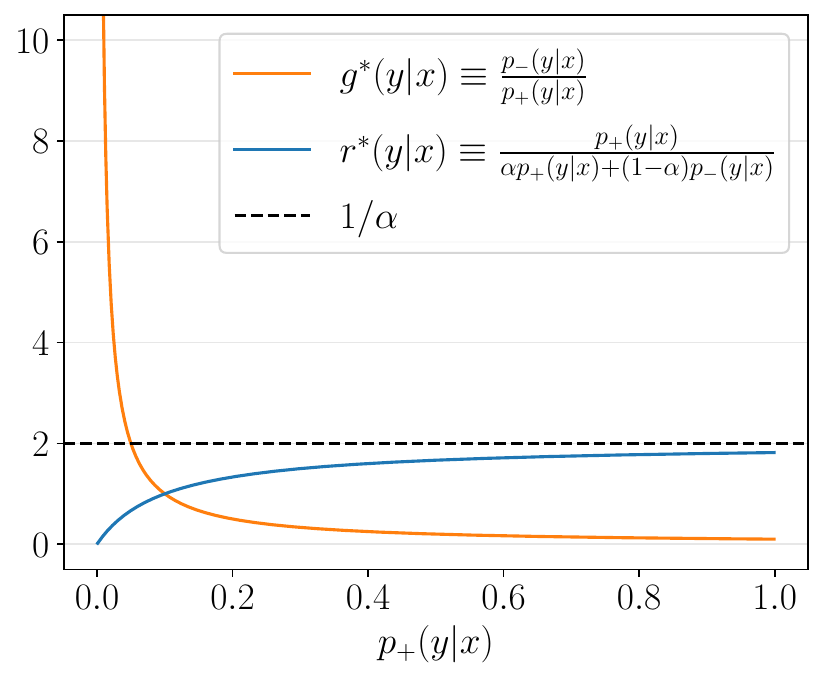}
    \caption{Comparing two ratios.}
    \label{fig:two_ratios}
  \end{subfigure}
  \begin{subfigure}{0.495\textwidth}
    \includegraphics[width=\linewidth]{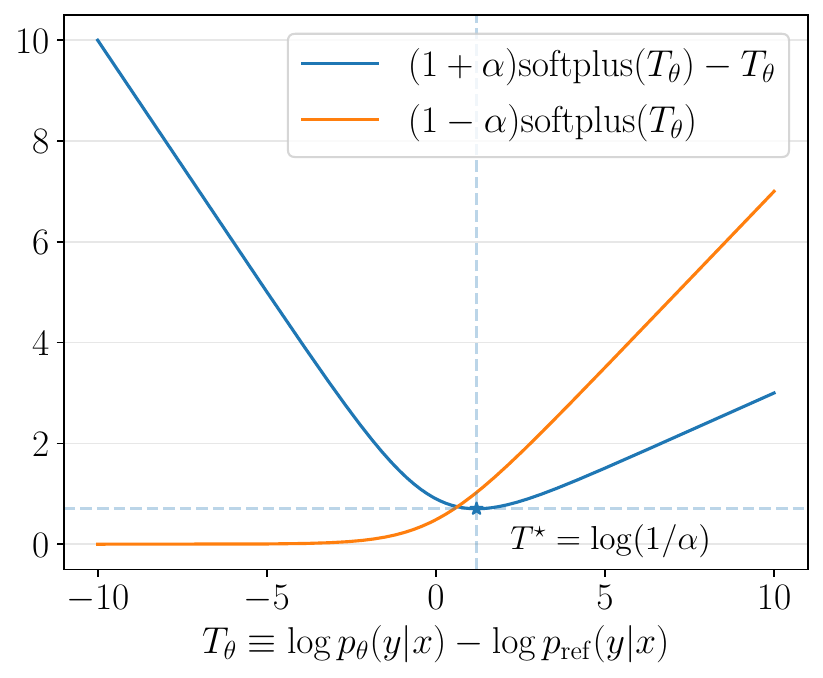}
    \caption{Loss functions.}
    \label{fig:loss_shape}
  \end{subfigure}
  \caption{Density ratio comparison and loss functions.}
\end{figure}

Figure~\ref{fig:two_ratios} compares the density ratio $g^{*}(y|x)$ and the relative density ratio $r^{*}(y|x)$,
where the non-preferred data distribution $p^{-}(y|x)=0.1$ and the hyperparameter $\alpha=0.5$.
Although $g^{*}(y|x)$ diverges as $p^{+}(y|x)\to 0$,
$r^{*}(y|x)$ is bounded above by $1/\alpha$.
Figure~\ref{fig:loss_shape} shows the loss functions for preferred (blue) and non-preferred (orange) samples with $\alpha=0.3$.
The loss for preferred samples is minimized when $T_{\theta} \equiv \log p_{\theta}(y|x) - \log p_{\mathrm{ref}}(y|x)$ attains $\log(1/\alpha)$.

\section{Pseudocode}
\label{sec:pseudocode}

Algorithm~\ref{alg:proposed} provides the pseudocode for RDRO.
\begin{algorithm}[h]
  \caption{Relative Density Ratio Optimization}
  \textbf{Input}: datasets $(\mathcal{D}^{+}, \mathcal{D}^{-})$, reference policy $p_{\mathrm{ref}}(y|x)$, and hyperparameter $\alpha\in(0,1)$\\
  \textbf{Output}: target policy $p_{\theta}(y|x)$\\
  \textbf{Procedure}:
  \begin{algorithmic}[1]
    \While{not converged}
      \State Sample mini-batch $\mathcal{B}$ from datasets $(\mathcal{D}^{+}, \mathcal{D}^{-})$
      \State Compute gradient of Eq.~(\ref{eq:proposed}) with $\mathcal{B}$
      \State Update $\theta$ with the above gradient
    \EndWhile
  \end{algorithmic}
  \label{alg:proposed}
\end{algorithm}

\section{Proof of Theorem \ref{thm:consistency}}
\label{sec:proof}

In this section, we will prove Theorem~\ref{thm:consistency} after providing four relevant lemmas.
For simplicity,
we denote the true risk $\mathcal{L}_{\mathrm{RDRE}}(\theta)$ as $\mathcal{L}(\theta)$,
and the empirical risk $\hat{\mathcal{L}}_{\mathrm{RDRE}}(\theta)$ as $\hat{\mathcal{L}}(\theta)$.

\begin{lemma}
$\mathcal{L}(\theta)$ is bounded below as follows:
\begin{align}
\mathcal{L}(\theta)=\mathbb{E}_{p_{\mathrm{ref}}}\left[\mathrm{Breg}_{f}\left(r^{*}\Vert r_{\theta}\right)\right]\geq\frac{\mu}{2}\mathbb{E}_{p_{\mathrm{ref}}}\left[\left(r^{*}-r_{\theta}\right)^{2}\right],
\end{align}
where $\mu$ satisfies $f^{\prime\prime}(t)\geq\mu$ for all $t$.
\label{lemma:1}
\end{lemma}
\begin{proof}
Since $f$ is a strongly convex function,
the following inequality holds:
\begin{align}
f(r^{*})\geq f(r_{\theta})+f^{\prime}(r_{\theta})(r^{*}-r_{\theta})+\frac{\mu}{2}\left(r^{*}-r_{\theta}\right)^{2},
\end{align}
and thus,
\begin{align}
\mathrm{Breg}_{f}\left(r^{*}\Vert r_{\theta}\right)
=f(r^{*})-f(r_{\theta})-f^{\prime}(r_{\theta})(r^{*}-r_{\theta})
\geq\frac{\mu}{2}\left(r^{*}-r_{\theta}\right)^{2}.
\end{align}
Taking expectations on both sides yields the claim.
\end{proof}

\begin{lemma}
The squared error between $p_{\theta}$ and $p^{+}$ can be bounded by $\mathcal{L}(\theta)$ as follows:
\begin{align}
\mathbb{E}_{p^{+}}\left[\left(p_{\theta}-p^{+}\right)^{2}\right]\leq\frac{2}{\alpha\mu}\mathcal{L}(\theta).
\end{align}
\label{lemma:2}
\end{lemma}
\begin{proof}
The squared error between $p_{\theta}$ and $p^{+}$ is bounded above as follows:
\begin{align}
\mathbb{E}_{p^{+}}\left[\left(p_{\theta}-p^{+}\right)^{2}\right]
&=\mathbb{E}_{p^{+}}\left[p_{\mathrm{ref}}^{2}\left(r_{\theta}-r^{*}\right)^{2}\right]\\
&=\mathbb{E}_{p_{\mathrm{ref}}}\left[r^{*}p_{\mathrm{ref}}^{2}\left(r_{\theta}-r^{*}\right)^{2}\right]\\
&\leq\frac{1}{\alpha}\mathbb{E}_{p_{\mathrm{ref}}}\left[p_{\mathrm{ref}}^{2}\left(r_{\theta}-r^{*}\right)^{2}\right]\\
&\leq\frac{1}{\alpha}\mathbb{E}_{p_{\mathrm{ref}}}\left[\left(r_{\theta}-r^{*}\right)^{2}\right],
\end{align}
where we used $p_{\theta}-p^{+}=p_{\mathrm{ref}}(r_{\theta}-r^{*})$,
$p^{+}=r^{*}p_{\mathrm{ref}}$,
$r^{*} \leq 1/\alpha$,
and $p_{\mathrm{ref}}\leq 1$ since $p_{\mathrm{ref}}$ is a discrete probability distribution.
Combining with Lemma~\ref{lemma:1} yields the claim.
\end{proof}

\begin{lemma}
For $\hat{\theta}$ that satisfies $\hat{\theta}=\argmin_{\theta}\hat{\mathcal{L}}(\theta)$,
$\mathcal{L}(\hat{\theta})$ is bounded above as follows:
\begin{align}
\mathcal{L}(\hat{\theta})\leq\inf_{\theta\in\Theta}\mathcal{L}(\theta)+\sup_{\theta\in\Theta}\left(\mathcal{L}(\theta)-\hat{\mathcal{L}}(\theta)\right)+\sup_{\theta\in\Theta}\left(\hat{\mathcal{L}}(\theta)-\mathcal{L}(\theta)\right).
\end{align}
\label{lemma:3}
\end{lemma}
\begin{proof}
$\mathcal{L}(\hat{\theta})$ is bounded above as follows:
\begin{align}
\mathcal{L}(\hat{\theta})
&=\mathcal{L}(\hat{\theta})+\mathcal{L}(\theta)-\mathcal{L}(\theta)+\hat{\mathcal{L}}(\hat{\theta})-\hat{\mathcal{L}}(\hat{\theta})+\hat{\mathcal{L}}(\theta)-\hat{\mathcal{L}}(\theta)\\
&=\mathcal{L}(\theta)+\left(\mathcal{L}(\hat{\theta})-\hat{\mathcal{L}}(\hat{\theta})\right)+\left(\hat{\mathcal{L}}(\theta)-\mathcal{L}(\theta)\right)+\left(\hat{\mathcal{L}}(\hat{\theta})-\hat{\mathcal{L}}(\theta)\right)\\
&\leq\mathcal{L}(\theta)+\sup_{\theta\in\Theta}\left(\mathcal{L}(\theta)-\hat{\mathcal{L}}(\theta)\right)+\sup_{\theta\in\Theta}\left(\hat{\mathcal{L}}(\theta)-\mathcal{L}(\theta)\right),
\end{align}
where we used $\hat{\mathcal{L}}(\hat{\theta})-\hat{\mathcal{L}}(\theta)\leq0$.
Taking the infimum over $\theta\in\Theta$ on both sides yields the claim.
\end{proof}

\begin{lemma}
The uniform deviation between the empirical risk $\hat{\mathcal{L}}(\theta)$ and the true risk $\mathcal{L}(\theta)$ is bounded above as follows:
\begin{align}
\mathbb{E}_{\mathcal{D}^{\pm}}\left[\sup_{\theta\in\Theta}\left(\hat{\mathcal{L}}(\theta)-\mathcal{L}(\theta)\right)\right]\leq2C_{\mathrm{Lip}}\left\{ \alpha\mathcal{R}_{N}(\mathcal{H})+(1-\alpha)\mathcal{R}_{M}(\mathcal{H})\right\},
\end{align}
where $\mathcal{D}^{\pm}=\left\{ \mathcal{D}^{+},\mathcal{D}^{-}\right\}$ is the dataset,
$C_{\mathrm{Lip}}$ is the Lipschitz constant of the Bregman divergence,
and $\mathcal{R}_{N}(\mathcal{H})$ and $\mathcal{R}_{M}(\mathcal{H})$ are Rademacher complexities of the model class $\mathcal{H}=\left\{ p_{\theta}\mid\theta\in\Theta\right\}$ with sample sizes $N$ and $M$, respectively.
\label{lemma:4}
\end{lemma}
\begin{proof}
Let $z=(x,y)$ and define $\ell_{\theta}(z)\equiv\mathrm{Breg}_{f}\left(r^{*}(z)\Vert r_{\theta}(z)\right)$.
With this notation,
we can rewrite $\mathcal{L}(\theta)$ and $\hat{\mathcal{L}}(\theta)$ as follows:
\begin{align}
\mathcal{L}(\theta)
&=\alpha\mathbb{E}_{p^{+}}\left[\ell_{\theta}(z)\right]+(1-\alpha)\mathbb{E}_{p^{-}}\left[\ell_{\theta}(z)\right],\\
\hat{\mathcal{L}}(\theta)
&=\alpha\frac{1}{N}\sum_{n=1}^{N}\ell_{\theta}(z_{n}^{+})+(1-\alpha)\frac{1}{M}\sum_{m=1}^{M}\ell_{\theta}(z_{m}^{-}),
\end{align}
where $\left(z_{n}^{+}\right)_{n=1}^{N}=\left(x_{n}^{+},y_{n}^{+}\right)_{n=1}^{N}=\mathcal{D}^{+}$ and $\left(z_{m}^{-}\right)_{m=1}^{M}=\left(x_{m}^{-},y_{m}^{-}\right)_{m=1}^{M}=\mathcal{D}^{-}$.
Then,
we obtain the following inequality:
\begin{multline}
\sup_{\theta\in\Theta}\left(\hat{\mathcal{L}}(\theta)-\mathcal{L}(\theta)\right)
\leq\alpha\sup_{\theta\in\Theta}\left(\frac{1}{N}\sum_{n=1}^{N}\ell_{\theta}(z_{n}^{+})-\mathbb{E}_{p^{+}}\left[\ell_{\theta}(z)\right]\right)\\
+(1-\alpha)\sup_{\theta\in\Theta}\left(\frac{1}{M}\sum_{m=1}^{M}\ell_{\theta}(z_{m}^{-})-\mathbb{E}_{p^{-}}\left[\ell_{\theta}(z)\right]\right).
\end{multline}
By the symmetrization trick \citep{mohri2018foundations} and the contraction inequality \citep{mohri2018foundations},
the uniform deviation is bounded above as follows:
\begin{align}
\mathbb{E}_{\mathcal{D}^{\pm}}\left[\sup_{\theta\in\Theta}\left(\hat{\mathcal{L}}(\theta)-\mathcal{L}(\theta)\right)\right]\leq2C_{\mathrm{Lip}}\left\{ \alpha\mathcal{R}_{N}(\mathcal{H})+(1-\alpha)\mathcal{R}_{M}(\mathcal{H})\right\},
\end{align}
where
\begin{align}
\mathcal{R}_{K}(\mathcal{H})=\mathbb{E}_{\mathcal{D}}\left[\mathbb{E}_{\sigma}\left[\sup_{\theta\in\Theta}\frac{1}{K}\sum_{k=1}^{K}\sigma_{k}p_{\theta}(z_{k})\right]\right]
\end{align}
is the Rademacher complexity of $\mathcal{H}$ with sample size $\left|\mathcal{D}\right|=K$,
and $\left(\sigma_{k}\right)_{k=1}^{K}$ are independent uniform random variables taking values in $\left\{ -1,+1\right\}$.
\end{proof}
In the same way, we can also obtain the following upper bound for the uniform deviation:
\begin{align}
\mathbb{E}_{\mathcal{D}^{\pm}}\left[\sup_{\theta\in\Theta}\left(\mathcal{L}(\theta)-\hat{\mathcal{L}}(\theta)\right)\right]
\leq2C_{\mathrm{Lip}}\left\{ \alpha\mathcal{R}_{N}(\mathcal{H})+(1-\alpha)\mathcal{R}_{M}(\mathcal{H})\right\}.
\end{align}

By combining Lemma~\ref{lemma:3} and Lemma~\ref{lemma:4}, we obtain the following upper bound:
\begin{align}
\mathbb{E}_{\mathcal{D}^{\pm}}\left[\mathcal{L}(\hat{\theta})\right]
&\leq\mathbb{E}_{\mathcal{D}^{\pm}}\left[\inf_{\theta\in\Theta}\mathcal{L}(\theta)\right]+\mathbb{E}_{\mathcal{D}^{\pm}}\left[\sup_{\theta\in\Theta}\left(\hat{\mathcal{L}}(\theta)-\mathcal{L}(\theta)\right)\right]+\mathbb{E}_{\mathcal{D}^{\pm}}\left[\sup_{\theta\in\Theta}\left(\mathcal{L}(\theta)-\hat{\mathcal{L}}(\theta)\right)\right]\\
&\leq\inf_{\theta\in\Theta}\mathcal{L}(\theta)+4C_{\mathrm{Lip}}\left\{ \alpha\mathcal{R}_{N}(\mathcal{H})+(1-\alpha)\mathcal{R}_{M}(\mathcal{H})\right\}.
\end{align}

By combining this with Lemma~\ref{lemma:2} evaluated at $\theta=\hat{\theta}$,
we obtain Theorem~\ref{thm:consistency}:
\begin{align}
\mathbb{E}_{\mathcal{D}^{\pm}}\left[\mathbb{E}_{p^{+}}\left[\left(p_{\hat{\theta}}-p^{+}\right)^{2}\right]\right]
&\leq\frac{2}{\alpha\mu}\mathbb{E}_{\mathcal{D}^{\pm}}\left[\mathcal{L}(\hat{\theta})\right]\\
&\leq\frac{2}{\alpha\mu}\left[\inf_{\theta\in\Theta}\mathcal{L}(\theta)+4C_{\mathrm{Lip}}\left\{ \alpha\mathcal{R}_{N}(\mathcal{H})+(1-\alpha)\mathcal{R}_{M}(\mathcal{H})\right\} \right].
\end{align}

\section{Details of the Convergence-Rate Comparison}
\label{sec:order_details}

Here we derive the condition on $\alpha$ used in Eq.~(\ref{eq:m_plus}).
Assuming $\alpha\in(0,1)$ and $\mu>0$, we have
\begin{align}
\frac{2}{\alpha\mu}<\frac{2\left(1-\alpha\right)^{2}}{\alpha^{2}m_{+}^{2}\mu}
\Leftrightarrow\alpha m_{+}^{2}<\left(1-\alpha\right)^{2}
\Leftrightarrow\alpha^{2}-(2+m_{+}^{2})\alpha+1>0,
\end{align}
which yields
\begin{align}
\alpha<\left(\frac{\sqrt{m_{+}^{2}+4}-m_{+}}{2}\right)^{2}.
\end{align}
Using a Taylor expansion around $m_{+}\rightarrow0$, we obtain
\begin{align}
\left(\frac{\sqrt{m_{+}^{2}+4}-m_{+}}{2}\right)^{2}=1-m_{+}+O(m_{+}^{2}),
\end{align}
and hence
\begin{align}
\alpha<1-m_{+}+O(m_{+}^{2}).
\label{eq:alpha_condition}
\end{align}
For any fixed $\alpha\in(0,1)$,
Eq.~(\ref{eq:alpha_condition}) holds for all sufficiently small $m_{+}$.
Therefore, in the regime $m_{+}\ll1$,
the right-hand side of Eq.~(\ref{eq:m_plus}) can be much larger than the left-hand side.

\section{Canonical Violation of Bradley-Terry Assumption}
\label{sec:cyclic_preference}

Under the Bradley-Terry assumption,
for two responses $y_{1}$ and $y_{2}$ to a prompt $x$, the probability that $y_{1}$ is preferred over $y_{2}$ is defined as follows:
\begin{align}
\mathrm{Pr}_{\phi}(y_{1}\succ y_{2}|x)&=\sigma(R_{\phi}(x,y_{1})-R_{\phi}(x,y_{2})),
\end{align}
where $R_{\phi}(x,y)$ denotes the reward model.
Now consider a prompt $x$ with three responses $y_{a}$, $y_{b}$, $y_{c}$ and assume a cyclic preference pattern:
\begin{align}
\mathrm{Pr}_{\phi}(y_{a}\succ y_{b}\mid x)
=\mathrm{Pr}_{\phi}(y_{b}\succ y_{c}\mid x)
=\mathrm{Pr}_{\phi}(y_{c}\succ y_{a}\mid x)
=t \neq \tfrac{1}{2}.
\end{align}
Under the Bradley-Terry assumption, this would require
\begin{align}
R_{\phi}(x,y_{a})-R_{\phi}(x,y_{b})
=R_{\phi}(x,y_{b})-R_{\phi}(x,y_{c})
=R_{\phi}(x,y_{c})-R_{\phi}(x,y_{a}).
\end{align}
Taking the sum of these three equalities yields $0$, and hence each difference must be zero.
Therefore, $\mathrm{Pr}(y_{a}\succ y_{b}\mid x)=\sigma(0)=1/2$,
contradicting the assumption $t \neq 1/2$.

\section{Additional Results on BIG-Bench Hard}
\label{sec:bbh}

\begin{table}[ht]
\caption{
    Comparison of BIG-Bench Hard (3 shot) on UF-G.
}
\begin{center}
\begin{tabular}{l|rrrr}
\toprule
 & Qwen-1.5B & Qwen-3B & Llama-3B & Llama-8B \\
\midrule
KTO  & 0.252 $\pm$ 0.000 & 0.000 $\pm$ 0.000 & 0.438 $\pm$ 0.000 & 0.190 $\pm$ 0.012 \\
DDRO & 0.279 $\pm$ 0.001 & 0.004 $\pm$ 0.001 & 0.531 $\pm$ 0.001 & 0.615 $\pm$ 0.018 \\
RDRO & \textbf{0.286 $\pm$ 0.000} & \textbf{0.010 $\pm$ 0.000} & \textbf{0.546 $\pm$ 0.001} & \textbf{0.701 $\pm$ 0.000} \\
\bottomrule
\end{tabular}
\label{tab:bbh_UF-G}
\end{center}
\end{table}

\begin{table}[ht]
\caption{
    Comparison of BIG-Bench Hard (3 shot) on MIX-14K.
}
\begin{center}
\begin{tabular}{l|rrrr}
\toprule
 & Qwen-1.5B & Qwen-3B & Llama-3B & Llama-8B \\
\midrule
KTO  & 0.243 $\pm$ 0.000 & \textbf{0.017 $\pm$ 0.000} & 0.538 $\pm$ 0.000 & 0.668 $\pm$ 0.001 \\
DDRO & 0.275 $\pm$ 0.001 & 0.015 $\pm$ 0.000 & \textbf{0.552 $\pm$ 0.001} & \textbf{0.699 $\pm$ 0.003} \\
RDRO & \textbf{0.284 $\pm$ 0.000} & 0.014 $\pm$ 0.000 & \textbf{0.553 $\pm$ 0.000} & \textbf{0.694 $\pm$ 0.001} \\
\bottomrule
\end{tabular}
\label{tab:bbh_MIX-14K}
\end{center}
\end{table}

As additional experiments,
we evaluate the aligned models on BIG-Bench Hard (BBH) \citep{suzgun2023challenging},
which is a benchmark for evaluating the reasoning capabilities of language models.
We used the \href{https://github.com/EleutherAI/lm-evaluation-harness}{lm-evaluation-harness},
where we set the number of fewshot examples to 3.

Tables \ref{tab:bbh_UF-G} and \ref{tab:bbh_MIX-14K} show the BBH results on UF-G and MIX-14K, respectively.
We used boldface to indicate the best result and any results not significantly different from it according to a pairwise t-test with a significance level of 5\%.

RDRO achieves the best or statistically comparable performance across all models and datasets,
with the sole exception of Qwen-3B on MIX-14K.
This indicates that RDRO performs equal to or better than DDRO across multiple benchmarks.
Note that for Qwen-3B, scores are generally low on BBH across all methods and datasets, including the base model before alignment,
suggesting that it is not well-suited to this benchmark.

\section{Empirical Analysis of Training Dynamics}
\label{sec:training_dynamics}

In this section,
we present the training dynamics over steps (within an epoch), including training losses (Figure \ref{fig:losses}),
gradient norms (Figure \ref{fig:norms}),
and log-ratio trajectories (Figures \ref{fig:logratios_chosen} and \ref{fig:logratios_rejected}).
Here, log-ratios are defined as the difference between the log-probabilities of the target policy and the reference policy.
Figure \ref{fig:logratios_chosen} shows the log-ratios for preferred data,
and Figure \ref{fig:logratios_rejected} shows the log-ratios for non-preferred data.
We use Llama-8B on UF-G with $\alpha=0.39$.

DDRO introduces heuristic stabilization (e.g., the softplus transformation described in Eq. (\ref{eq:stabilized_dre})) to prevent gradient explosion,
as also observed in Figure 3 of \citep{higuchi2025direct}.
As a result, gradient norms appear small and training seems stable.
However, these modifications change the original objective and may lead to a mismatch with its theoretical properties.
Indeed, DDRO exhibits large changes in both preferred and non-preferred log-ratios,
indicating overly large deviations from the reference policy.
Prior work has shown that overly large deviations from the reference policy can degrade performance \citep{kwa2024catastrophic},
and our observations are consistent with this phenomenon.
While one might expect that tuning the KL regularization parameter $\beta$ could mitigate this issue,
Appendix D of \citep{higuchi2025direct} explicitly states that the KL regularization is not included in the gradient calculation for training stability.
This suggests that controlling such deviations within DDRO is non-trivial.

In contrast, RDRO does not rely on heuristic stabilization.
The softplus transformation employed in DDRO arises naturally in our formulation,
while preserving statistical consistency.
As shown in Figure~\ref{fig:norms},
although RDRO may exhibit larger gradient norms than heuristic-stabilized DDRO or KTO,
it avoids the gradient explosion observed in the original DDRO formulation.
Moreover, RDRO avoids overly aggressive updates in both preferred and non-preferred log-ratios,
while consistently improving their separation (log-ratio margin), indicating controlled alignment behavior.

Overall, RDRO may appear less stable if stability is interpreted as having small gradient norms.
However, we argue that stability should instead be understood as
(i) avoiding gradient explosion and
(ii) preventing overly large deviations from the reference policy.
Under this notion, RDRO provides more principled and stable optimization behavior,
consistent with both our theoretical guarantees and empirical observations.

\begin{figure}[p]
\centering
\begin{subfigure}{0.32\textwidth}
    \includegraphics[width=\linewidth]{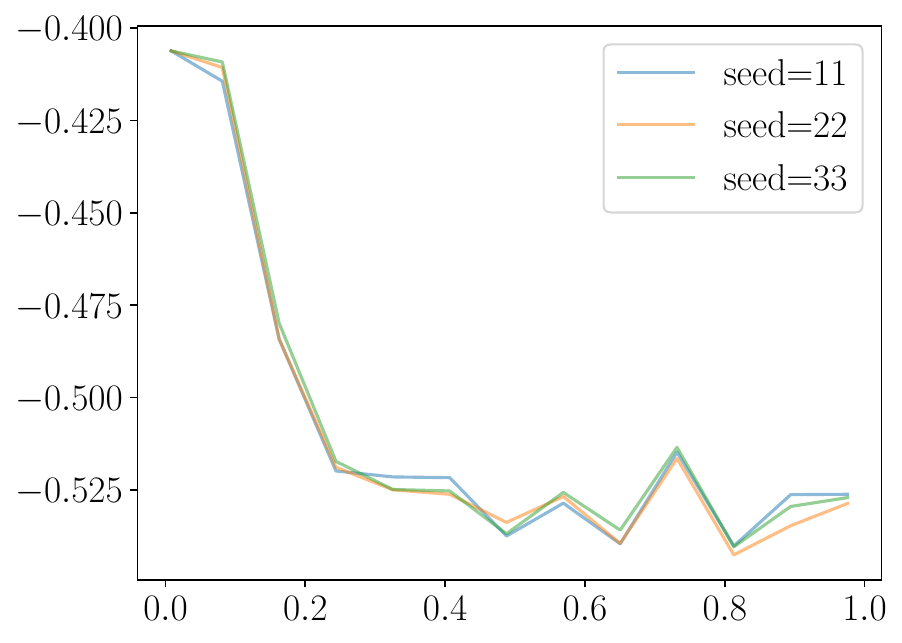}
    \caption{DDRO.}
\end{subfigure}
\begin{subfigure}{0.32\textwidth}
    \includegraphics[width=\linewidth]{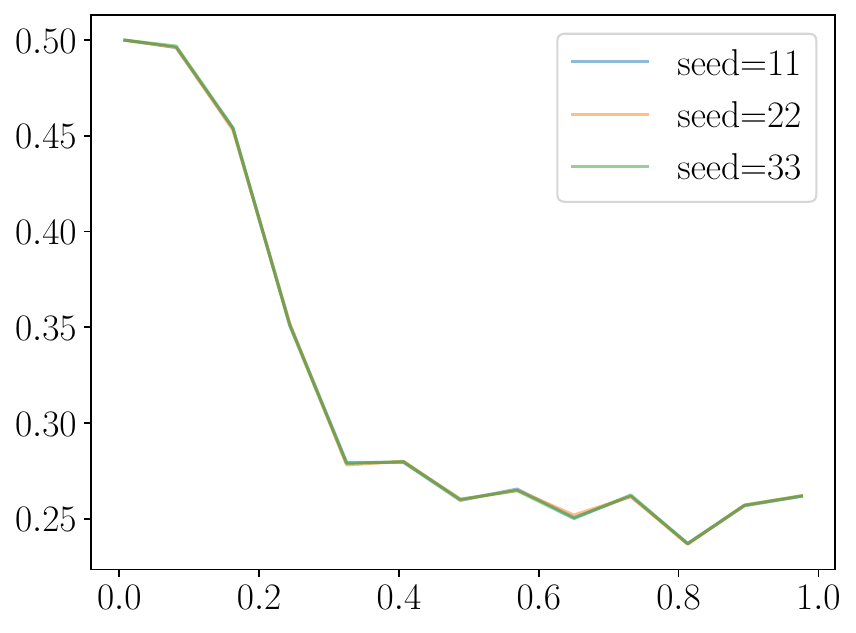}
    \caption{KTO.}
\end{subfigure}
\begin{subfigure}{0.32\textwidth}
    \includegraphics[width=\linewidth]{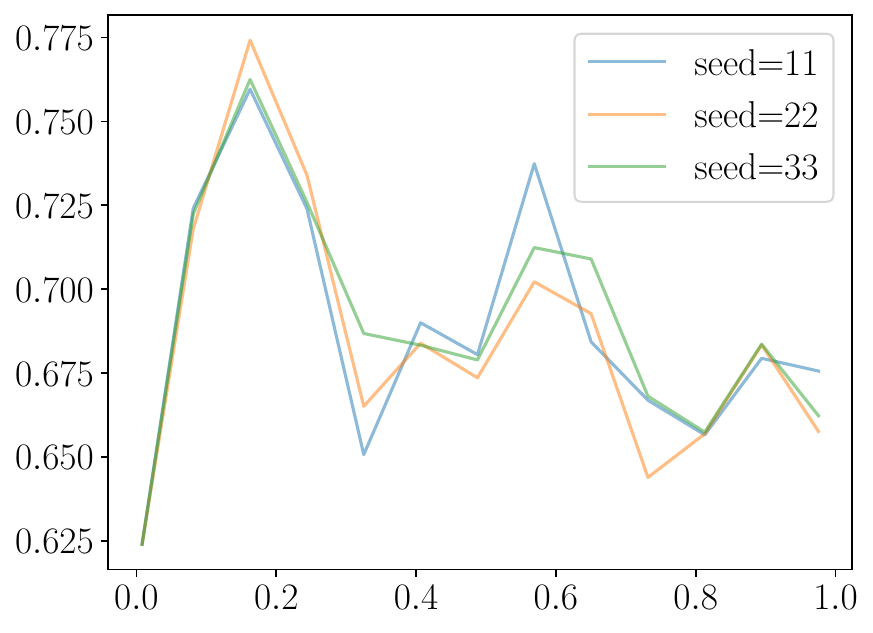}
    \caption{RDRO.}
\end{subfigure}
\caption{
    Training losses over steps for DDRO, KTO, and RDRO on Llama-8B with UF-G and $\alpha=0.39$.
    Note that the loss values are not directly comparable across methods since the objectives differ.
}
\label{fig:losses}
\end{figure}

\begin{figure}[p]
\centering
\begin{subfigure}{0.32\textwidth}
    \includegraphics[width=\linewidth]{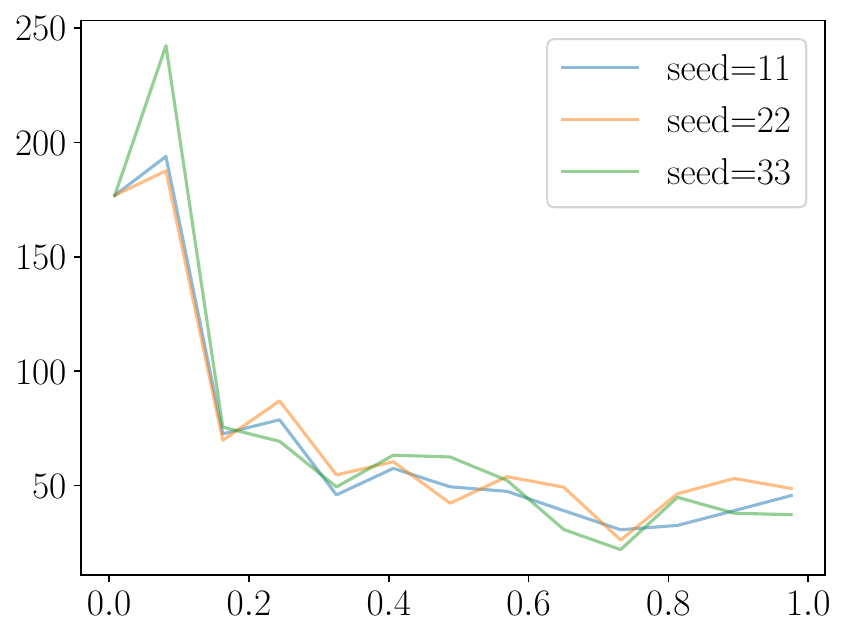}
    \caption{DDRO.}
\end{subfigure}
\begin{subfigure}{0.32\textwidth}
    \includegraphics[width=\linewidth]{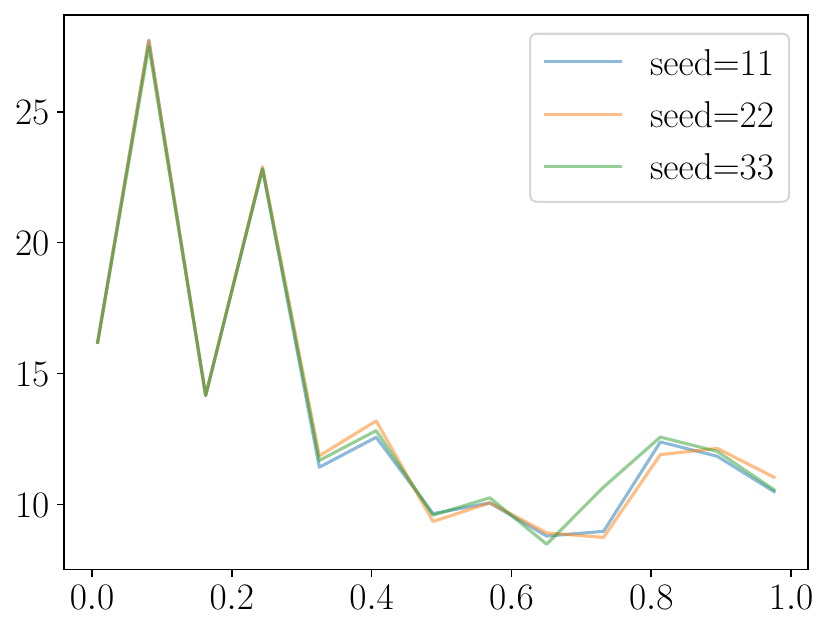}
    \caption{KTO.}
\end{subfigure}
\begin{subfigure}{0.32\textwidth}
    \includegraphics[width=\linewidth]{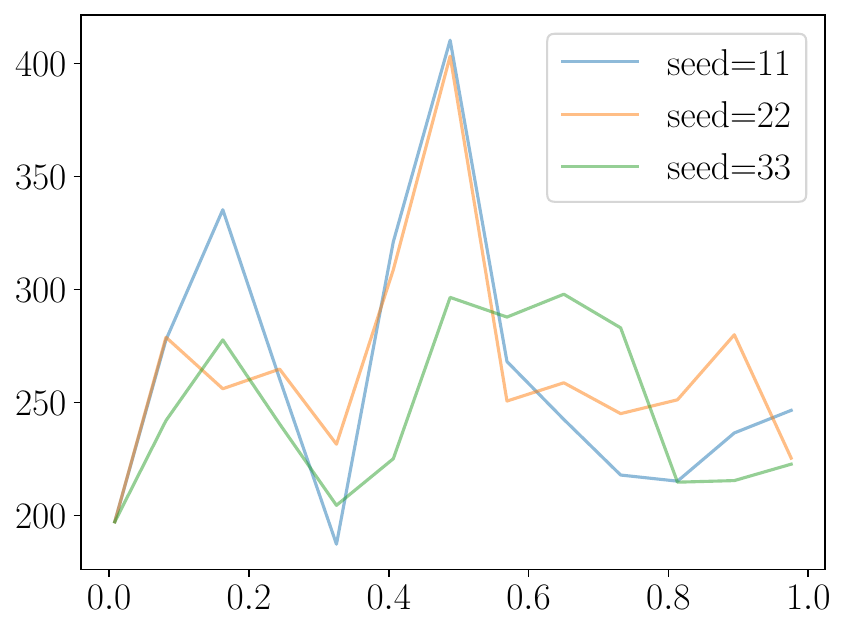}
    \caption{RDRO.}
\end{subfigure}
\caption{
    Gradient norms over steps for DDRO, KTO, and RDRO on Llama-8B with UF-G and $\alpha=0.39$.
}
\label{fig:norms}
\end{figure}

\begin{figure}[p]
\centering
\begin{subfigure}{0.32\textwidth}
    \includegraphics[width=\linewidth]{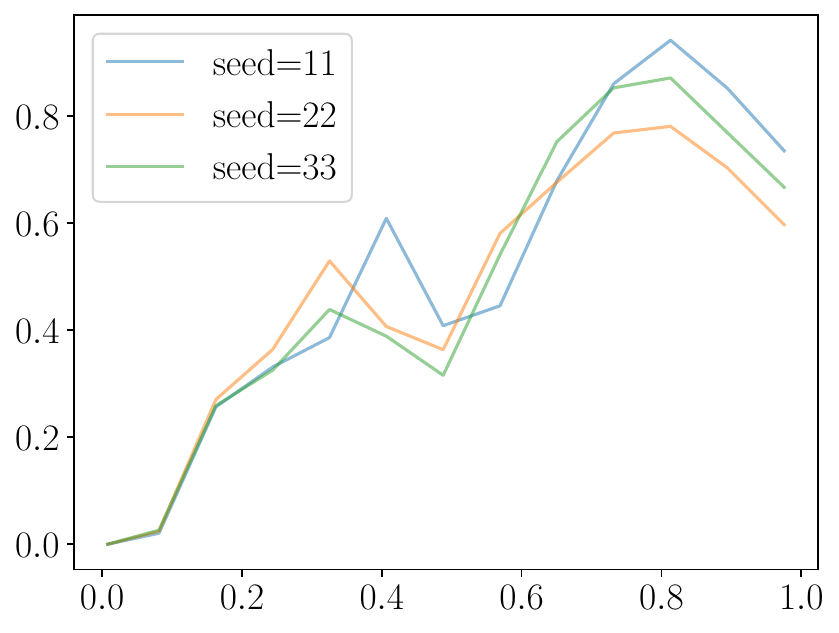}
    \caption{DDRO.}
\end{subfigure}
\begin{subfigure}{0.32\textwidth}
    \includegraphics[width=\linewidth]{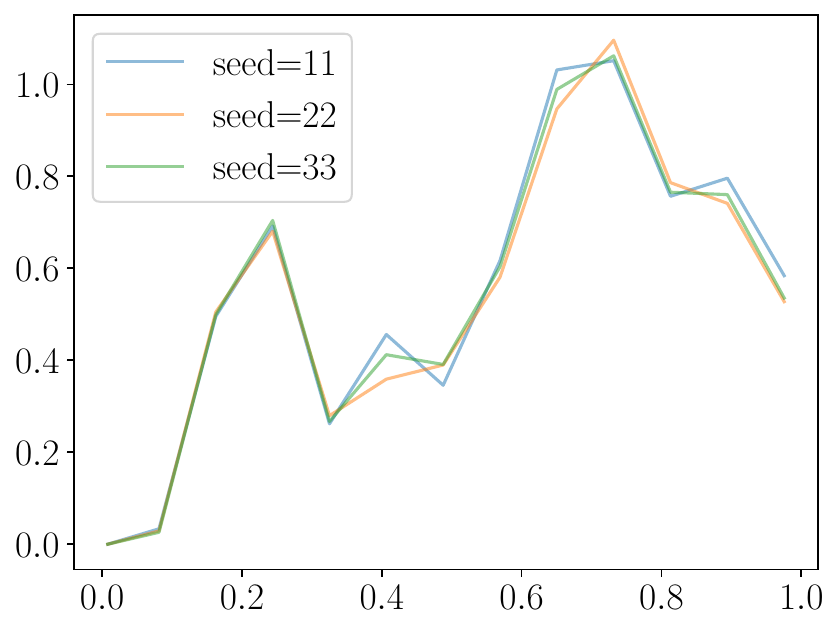}
    \caption{KTO.}
\end{subfigure}
\begin{subfigure}{0.32\textwidth}
    \includegraphics[width=\linewidth]{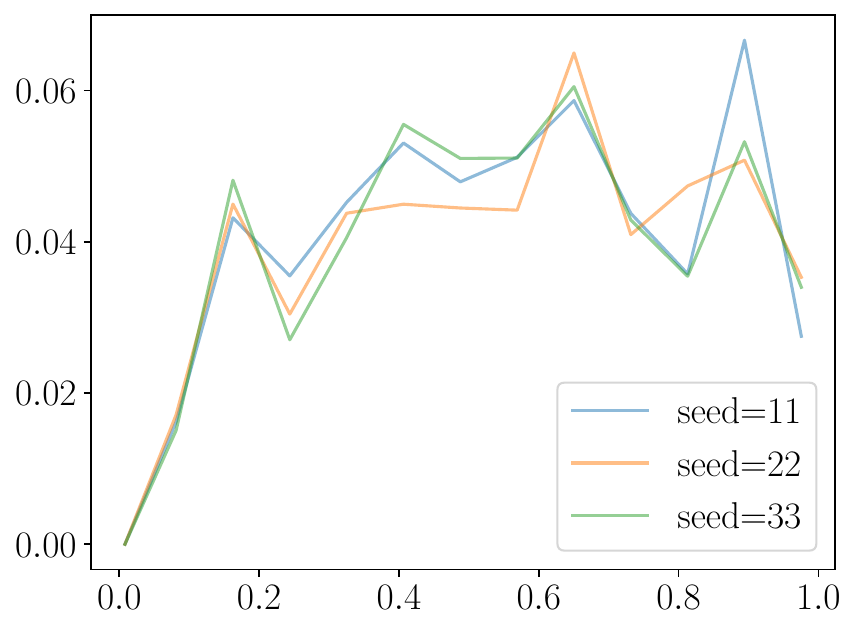}
    \caption{RDRO.}
\end{subfigure}
\caption{
  Preferred log-ratios over steps for DDRO, KTO, and RDRO on Llama-8B with UF-G and $\alpha=0.39$.
}
\label{fig:logratios_chosen}
\end{figure}

\begin{figure}[p]
\centering
\begin{subfigure}{0.32\textwidth}
    \includegraphics[width=\linewidth]{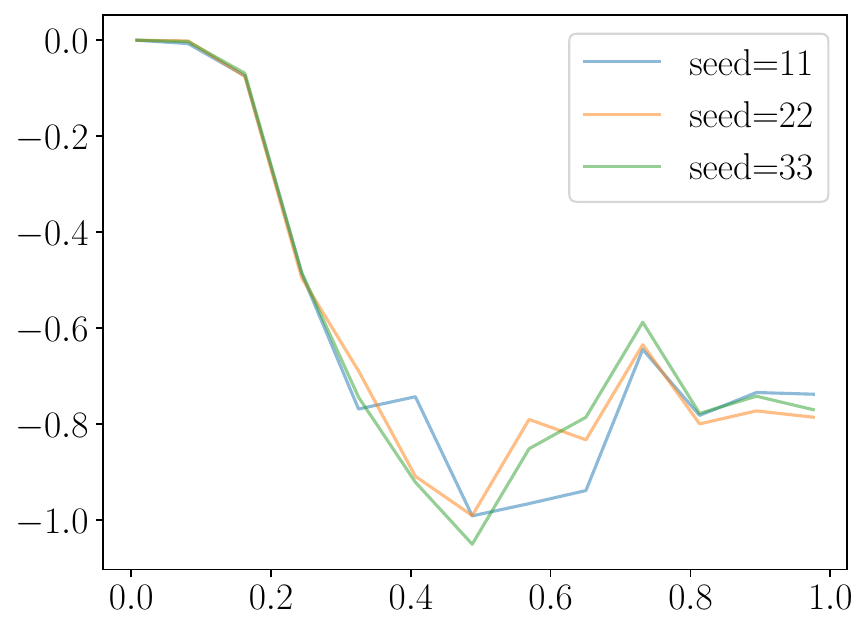}
    \caption{DDRO.}
\end{subfigure}
\begin{subfigure}{0.32\textwidth}
    \includegraphics[width=\linewidth]{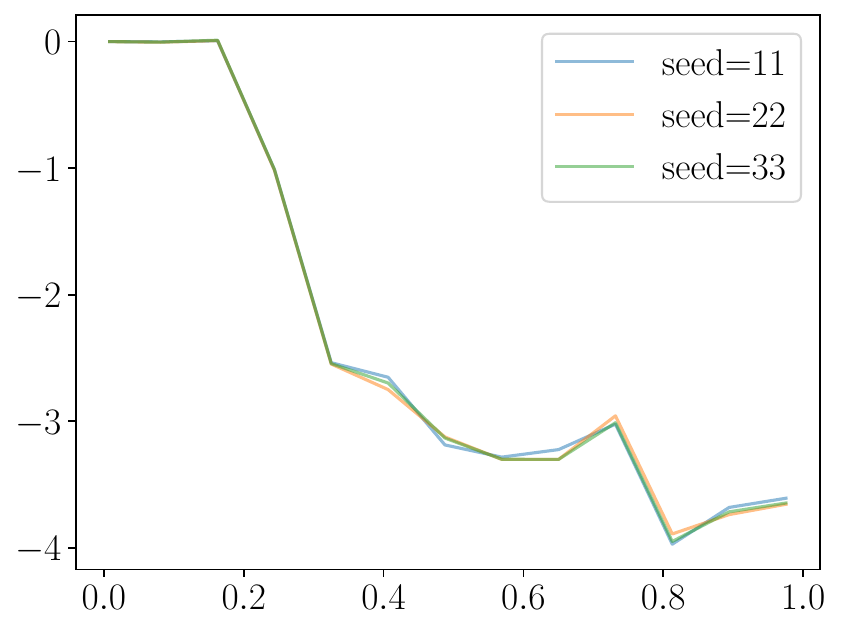}
    \caption{KTO.}
\end{subfigure}
\begin{subfigure}{0.32\textwidth}
    \includegraphics[width=\linewidth]{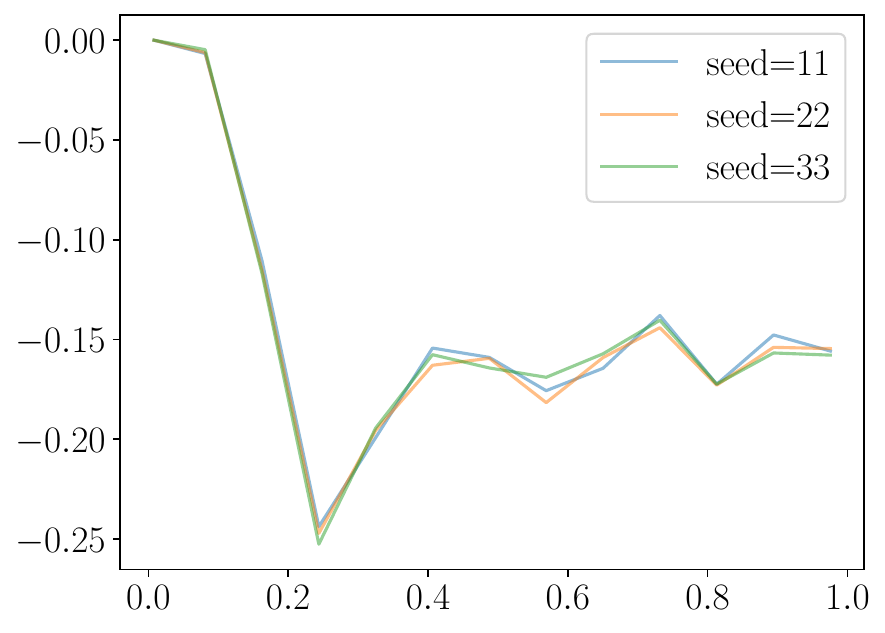}
    \caption{RDRO.}
\end{subfigure}
\caption{
  Non-preferred log-ratios over steps for DDRO, KTO, and RDRO on Llama-8B with UF-G and $\alpha=0.39$.
}
\label{fig:logratios_rejected}
\end{figure}

\end{document}